\pdfoutput=1
\documentclass{article}

\PassOptionsToPackage{round}{natbib}

\usepackage[main, final]{neurips_2026}
\usepackage{enumitem}
\usepackage[utf8]{inputenc}
\usepackage[T1]{fontenc}
\usepackage{url}
\usepackage{booktabs}
\usepackage{amsfonts}
\usepackage{amsmath}
\usepackage{amssymb}
\usepackage{nicefrac}
\usepackage{microtype}
\usepackage{xcolor}
\usepackage{colortbl}
\usepackage{multirow}
\definecolor{anchorrow}{RGB}{233,242,253}
\definecolor{citeorange}{RGB}{200,90,30}
\usepackage{graphicx}
\graphicspath{{figures/}}
\usepackage{algorithm}
\usepackage{algorithmic}
\usepackage{float}
\usepackage{mathtools}
\usepackage{amsthm}
\usepackage[colorlinks=true,linkcolor=black,citecolor=citeorange,urlcolor=blue]{hyperref}
\usepackage{cleveref}
\usepackage{xcolor}
\definecolor{ssmblue}{HTML}{1F77B4}

\DeclareMathOperator*{\argmin}{arg\,min}
\DeclareMathOperator*{\argmax}{arg\,max}

\newcommand{\E}{\mathbb{E}}
\newcommand{\R}{\mathbb{R}}
\newcommand{\Qr}{Q_r}
\newcommand{\Qh}{Q_h}
\newcommand{\Vh}{V_h}

\newcommand{\calS}{\mathcal{S}}
\newcommand{\calA}{\mathcal{A}}
\newcommand{\calB}{\mathcal{B}}

\newcommand{\Isafe}{\mathbf{1}_{\mathrm{safe}}}
\newcommand{\Iind}{\mathbf{1}_{\mathrm{ind}}}

\newcommand{\gqr}{\nabla_{\!a} \Qr}
\newcommand{\gqh}{\nabla_{\!a} \Qh}

\newcommand{\Hent}{\mathcal{H}}
\newcommand{\indicator}[1]{\mathbf{1}{\left\{#1\right\}}}
\newcommand{\calL}{\mathcal{L}}
\newcommand{\calC}{\mathcal{C}}
\newtheorem{theorem}{Theorem}
\newtheorem{proposition}[theorem]{Proposition}

\newtheorem{definition}{Definition}

\newtheorem{problem}{Problem}

\definecolor{bestfont}{RGB}{0,70,168}
\definecolor{secondfont}{RGB}{56,166,0}
\definecolor{infeasiblebg}{RGB}{250,224,224}
\newcommand{\bestR}[1]{\textcolor{bestfont}{$\boldsymbol{#1}$}}
\newcommand{\secondR}[1]{\textcolor{secondfont}{$\boldsymbol{#1}$}}
\newcommand{\ob}[1]{\cellcolor{infeasiblebg}$#1$}

\title{Safe Score Matching: Diffusion Policies with \\Hamilton-Jacobi Reachability for Online Safe Reinforcement Learning}

\author{%
  Boyang Li\thanks{Equal contribution.}\\
  University of California San Diego\\
  \texttt{bol025@ucsd.edu}
  \And
  Matthew Kim\footnotemark[1]\\
  University of California San Diego\\
  \texttt{mak009@ucsd.edu}
  \And
  Sylvia Lee Herbert\\
  University of California San Diego\\
  \texttt{sherbert@ucsd.edu}
}

\begin{document}
\maketitle

\begin{abstract}
Online safe reinforcement learning (RL) seeks policies that maximize reward while satisfying safety constraints. A popular line of research in safe RL relaxes safety to a soft expected-cost constraint and solves the resulting Constrained Markov Decision Process (CMDP) via primal-dual Lagrangian updates that only enforce safety on average. To address this limitation, hard, state-wise constraints are introduced and often imposed through Hamilton-Jacobi (HJ) reachability. Yet such constraints require solving different objectives in the feasible and infeasible regions of the state space: reward maximization in the former, recovery toward the feasible regions in the latter. The resulting target action distributions are inherently multimodal, and this structure poses a fundamental challenge for the Gaussian or deterministic actors used in existing HJ-based safe RL, which often collapse onto suboptimal modes. Diffusion policies provide the expressiveness needed to represent such distributions, and recent work on Q-score matching offers a route to training them for online RL by score regression---but has been applied only to reward maximization. Building on this framework, we propose Safe Score Matching (SSM), an off-policy actor-critic method that adapts Q-score matching to hard-constrained safe RL by gating a two-branch score target with HJ reachability: inside the feasible set, the denoising process degenerates to Q-score matching on actions classified as viable by the HJ critic, encouraging reward maximization; outside, a recovery branch biases denoising toward regions with lower worst-case violation, guided by the negative action gradient of the HJ safety value. On quadrotor and fixed-wing trajectory-tracking and stabilize-and-avoid benchmarks, SSM attains the best or near-best task performance with low false-safe rates, whereas the primal-dual baseline admits more unsafe behavior and reachability-based baselines tend to be more conservative; on Safety-Gymnasium velocity tasks, SSM attains the lowest cost with competitive reward. Code is available at \url{https://github.com/byli888/safe-score-matching}.
\end{abstract}

\section{Introduction}
\label{sec:intro}

Reinforcement learning policies that operate in physical systems must
respect safety requirements whose violation has no statistical
remedy~\citep{garcia2015comprehensive, gu2024review}. The standard
formulation is the constrained Markov decision process (CMDP), in
which the agent maximizes a reward objective subject to a constraint
on safety violations~\citep{altman2021constrained, achiam2017constrained}.
Two families of solvers dominate the online setting; each leaves a
gap that motivates the present work.

Primal--dual methods relax the constraint to an expected
cumulative-cost budget and solve the resulting Lagrangian saddle
point through alternating
updates~\citep{tessler2018reward, ray2019benchmarking, stooke2020responsive, wu2024offpolicy}.
The relaxation is convenient but loses the state-wise guarantee:
even at the optimum, individual trajectories can violate the
constraint~\citep{yu2022reachability, zheng2024safe}. A second family
takes the opposite stance. Hamilton--Jacobi (HJ) reachability
characterizes the largest control-invariant set
$\calS_f^\star = \{s : \Vh^\star(s) \le 0\}$ on which the hard
constraint $h \le 0$ is enforceable, and it carries a built-in
recovery target on the
complement~\citep{fisac2019bridging, hsu2021safety, yu2022reachability, ma2021feasible, ganai2023iterative}.
HJ delivers a hard, state-wise formulation, but the geometry of its
feasible action set
$\calA_f^\star(s) = \{a : \Qh^\star(s, a) \le 0\}$ is non-convex and
in general disconnected, and the policies extracted in prior work
are Gaussian or deterministic. A unimodal actor cannot place mass on
several disconnected viable components without leaking into the
unsafe region between them, and even on a single connected component
it is prone to collapsing onto an early-discovered local optimum of
the reward
$Q$-function~\citep{li2024learning, dong2025maximum, ding2024diffusion, algd2026}.

A separate line of work has shown that diffusion policies are
expressive enough to represent multimodal action distributions in
online RL. Q-score matching (QSM)~\citep{psenka2023learning} aligns
a DDPM denoiser to the action gradient of the reward $Q$-function;
subsequent score- and flow-matching methods follow the same
template~\citep{ma2025soft, lv2025flow, zhang2026sac, mcallister2026flow}.
These methods are reward-driven and provide no mechanism for hard,
state-wise safety.

\begin{table}[h]
\centering
\small
\caption{Positioning of SSM by design properties. SSM combines online actor--critic training, state-wise HJ safety, a diffusion-policy class, and an HJ-gated score-matching actor update.}
\label{tab:positioning}
\setlength{\tabcolsep}{4pt}
\begin{tabular}{@{}lcccc@{}}
\toprule
Method family & Online & Safety semantics & Policy class & Actor update \\
\midrule
Primal--dual safe RL & $\checkmark$ & Expected cost & Gaussian / deterministic & Primal--dual PG \\
HJ / reachability safe RL & $\checkmark$ & State-wise / HJ & Gaussian / deterministic & Actor--critic / shielded \\
Reward-only diffusion RL & $\checkmark$ & Reward only & Diffusion / flow & Score / flow matching \\
Offline safe diffusion & --- & State-wise / HJ & Diffusion & Offline guided regression \\
\textbf{SSM (Ours)} & $\checkmark$ & State-wise / HJ & Diffusion & HJ-gated score matching \\
\bottomrule
\end{tabular}
\end{table}

We close the gap between these two lines. Starting from the
reachability-constrained formulation of safe
RL~\citep{yu2022reachability, zheng2024safe}, we route an
action-level target distribution by HJ feasibility: on feasible
states with viable actions the target follows the reward gradient,
and on infeasible states it follows the recovery direction that
minimizes the worst-case safety violation. The routing depends only
on the learned safety critic, and the resulting score field is a
function of the learned reward and safety critics alone. The
diffusion policy is trained to match this score field. We call the
resulting algorithm \textbf{Safe Score Matching}
(SSM);
Table~\ref{tab:positioning} positions it relative to existing safe RL
families. The constraint is hard at the level of the formulation. Because SSM learns its critics from data, we evaluate the safety of
the trained policy empirically, and Proposition~\ref{prop:gate-margin}
states how a bounded critic error affects the routing. Our
contributions are threefold:
\begin{itemize}[leftmargin=*, itemsep=3pt, topsep=3pt]
    \item \textbf{Unified framework for safe online diffusion
    policy learning.} We introduce Safe Score Matching (SSM), which, to the best of our knowledge, is the first online actor-critic method that combines HJ reachability-based state-wise feasibility with diffusion-policy score regression for hard-constrained safe RL.

    \item \textbf{HJ-gated score regression target.} We derive a
    closed-form, action-level score regression target whose support
    and energy are routed by an HJ feasibility gate. On feasible
    states with viable actions, the target tracks the action
    gradient of the reward critic; on infeasible states, the
    recovery direction. The gate is determined by the learned safety
    critic, and the entire target depends only on standard off-policy
    critics, separating the safety routing from the policy class.

    \item \textbf{Empirical validation.} On Quad2D trajectory
    tracking, Quad3D regulation, and F16 stabilize-and-avoid, SSM
    attains the best or near-best task metric on each benchmark with low false-safe rates. On two Safety-Gymnasium velocity tasks, it attains the lowest cost of all compared methods with competitive
    reward.
\end{itemize}

\section{Related Work}
\label{sec:related}

SSM sits at the intersection of three lines of work in safe RL: (i) online safe RL methods that enforce safety as an expected cumulative-cost budget through primal--dual or Lagrangian updates~\citep{achiam2017constrained, ray2019benchmarking, stooke2020responsive, wu2024offpolicy}, which include both on-policy CPO/PPO-Lag/PID-Lag and off-policy SAC-Lag/CVPO/CAL variants but enforce safety only on average; (ii) state-wise / reachability-based methods that constrain the policy to a control-invariant set, including Lyapunov-based methods~\citep{chow2018lyapunov}, control barrier functions~\citep{qin2022sablas, ma2021model}, and Hamilton--Jacobi reachability analysis~\citep{fisac2019bridging, hsu2021safety, yu2022reachability, ma2021feasible, ganai2023iterative, qin2024feasible, li2024dual, sharpless2026dualobjective}, which target hard, state-wise safety but rely on Gaussian or deterministic actors; and (iii) generative-policy methods, which include offline diffusion / flow policies~\citep{janner2022planning, wang2022diffusion, kang2023efficient, ren2024diffusion} and online score- and flow-matching methods~\citep{psenka2023learning, ma2025soft, lv2025flow, zhang2026sac, mcallister2026flow, li2026reverse}, all reward-only. The closest prior work is FISOR~\citep{zheng2024safe}, which combines hard safety, HJ-style feasibility, and a diffusion policy in the offline setting. SSM studies the corresponding online actor--critic setting, where the diffusion policy must be updated from newly collected data and the HJ-derived signal must be converted into a denoising score target. An extended discussion on related works appears in Appendix~\ref{app:related}.

\section{Problem Formulation}
\label{sec:prelim}
\label{sec:cmdp}

Consider a Markov decision process $(\calS, \calA, F, r, h, c, \gamma_r, \gamma_h)$ with state space $\calS$, action space $\calA$, deterministic dynamics $F : \calS \times \calA \to \calS$ ($s_{t + 1} = F(s_t, a_t)$), reward $r : \calS \times \calA \to \R$, constraint function $h : \calS \to \R$ ($h(s) > 0$ indicates violation), cost $c(s) \triangleq \max(h(s), 0) \in [0, C_{\text{max}}]$, and discount factors $\gamma_r, \gamma_h \in (0, 1)$. Given a stationary policy $\pi$, we write $\{s_t^\pi\}_{t \geq 0}$ for the trajectory induced by $\pi$ from $s_0^\pi = s_0$. The reward value and action-value functions are
\begin{align}
\label{eq:reward-values}
V_r^\pi(s) &= \E_\pi\!\left[\textstyle\sum_{t=0}^\infty \gamma_r^t\, r(s_t^\pi, a_t) \,\Big|\, s_0 = s\right], \notag \\
\Qr^\pi(s, a) &= \E_\pi\!\left[\textstyle\sum_{t=0}^\infty \gamma_r^t\, r(s_t^\pi, a_t) \,\Big|\, s_0 = s,\ a_0 = a\right],
\end{align}
and the standard value function for cost is $V_c^\pi(s) = \E_\pi[\sum_{t=0}^\infty \gamma_h^t c(s_t^\pi) \mid s_0 = s]$.

\paragraph{Soft cumulative-cost formulation and its limitations.}
Most online safe RL methods solve the CMDP relaxation
\begin{equation}
\label{eq:soft-cmdp}
\max_\pi \; \E_{s_0 \sim \rho_0}\!\left[V_r^\pi(s_0)\right]
\quad \text{s.t.} \quad
\E_{s_0 \sim \rho_0}\!\left[V_c^\pi(s_0)\right] \leq l,
\end{equation}
typically by primal--dual Lagrangian updates~\citep{chow2018lyapunov, tessler2018reward, ray2019benchmarking, stooke2020responsive}. The constraint is aggregate over initial states and trajectories, so even at the optimum it can be satisfied while individual trajectories violate the pointwise condition $h \le 0$~\citep{yu2022reachability, zheng2024safe}; the optimal $l$ also varies across tasks~\citep{ji2023safety}.

\paragraph{Hard state-wise constraint.}
To remove these limitations at the formulation level, we replace~\eqref{eq:soft-cmdp} by its pointwise hard-constraint version, in which $h$ must be non-positive at every step, and define the following problem that our Safe Score Matching (SSM) aims to solve:

\begin{problem}[Hard-constrained state-wise problem]
\label{prob:hard-constrained}
\begin{align}
\max_\pi \quad & \E_{s_0 \sim \rho_0}\!\left[V_r^\pi(s_0)\right] \label{eq:cmdp} \\
\text{s.t.} \quad & h(s_t^\pi) \leq 0, \quad \forall\, t \geq 0,\;\;\rho_0\text{-a.s.\ }s_0. \nonumber
\end{align}
\end{problem}

\subsection{Reachability Reformulation}
\label{sec:reachability}
\label{sec:feasibility}

\noindent Problem~\ref{prob:hard-constrained} is strictly stronger than~\eqref{eq:soft-cmdp}, but its trajectory-level quantifier $\forall\, t \geq 0$ is not directly amenable to temporal-difference learning, and infeasible initial states leave it unsatisfiable. Hamilton--Jacobi (HJ) reachability analysis~\citep{fisac2019bridging, hsu2021safety, yu2022reachability} addresses both: it produces value functions that reason only about initial states, and naturally separates reward maximization on the feasible region from recovery on its complement.

Consider the following (optimal) value function that captures the worst-case violation along a trajectory.

\begin{definition}[HJ Safety Value Functions]
\label{def:reachability-values}
For a stationary policy $\pi$, the HJ safety value function $V_h^\pi$ for $\pi$ and the optimal HJ safety value function $V_h^\star$ are
\begin{equation}
\label{eq:vh-pi}
V_h^\pi(s_0) \triangleq \max_{t \geq 0} h(s_t^\pi),
\qquad
V_h^\star(s_0) \triangleq \min_\pi V_h^\pi(s_0),
\qquad s_0^\pi = s_0.
\end{equation}
\end{definition}

By construction, $V_h^\pi(s_0) \leq 0$ holds if and only if $h(s_t^\pi) \le 0$ for all $t \ge 0$, so $V_h$ collapses the infinite-horizon trajectory constraint of Problem~\ref{prob:hard-constrained} into a constraint on the initial state. Similarly, $V_h^\star(s_0) \leq 0$ holds if and only if there exists an optimal policy $\pi$ that enforces hard constraints. We next define the following (optimal) feasible region:

\begin{definition}[Feasible regions]
\label{def:feasible-region}
The feasible region of $\pi$ and the optimal feasible region are $\calS_f^\pi \triangleq \{s \in \calS : V_h^\pi(s) \le 0\}$ and $\calS_f^\star \triangleq \{s \in \calS : V_h^\star(s) \le 0\}$.
\end{definition}

We also need the corresponding state-action quantities, which capture the worst-case violation along the trajectory rolled out from an initial state--action pair.
\begin{definition}[HJ Avoid State-Action Value Functions]
\label{def:qh}
For a stationary policy $\pi$, the HJ avoid state-action value function $\Qh^\pi$ for $\pi$ and the optimal HJ avoid state-action value function $\Qh^\star$ are
\begin{equation}
\label{eq:qh-pi}
\Qh^\pi(s, a) \triangleq \max_{t \geq 0} h(s_t^\pi),
\qquad
\Qh^\star(s, a) \triangleq \min_\pi \Qh^\pi(s, a),
\end{equation}
where $\{s_t^\pi\}_{t \geq 0}$ is the trajectory rolled out from $(s_0^\pi, a_0) = (s, a)$ under $\pi$.
\end{definition}

\noindent Under the deterministic dynamics and attainment assumptions stated above, $\Qh^\star$ satisfies the self-consistency relation
\begin{equation}
\label{eq:qh-sbe}
\Qh^\star(s, a) = \max\!\big\{h(s),\; V_h^\star(F(s, a))\big\},
\qquad
V_h^\star(s) = \min_{a \in \calA} \Qh^\star(s, a).
\end{equation}

\begin{definition}[Optimal Feasible Action Set]
\label{def:Af}
At each state $s$, the optimal feasible action set is $\calA_f^\star(s) \triangleq \{a \in \calA : \Qh^\star(s, a) \leq 0\}$.
\end{definition}

\noindent By~\eqref{eq:qh-sbe}, any stochastic policy whose support lies in $\calA_f^\star(s)$ at every feasible state keeps $\calS_f^\star$ forward-invariant by induction~\citep{yu2022reachability,zheng2024safe}; the learned implementation in Section~\ref{sec:method} approximates this support-restriction mechanism and is not by itself an oracle safety guarantee.

\paragraph{Practical $V_h^\star$ approximation.}
For cost discount $\gamma_h \in (0, 1)$, the discounted HJ Bellman operator
\begin{equation}
\label{eq:hj_bellman}
\mathcal{P}^\star Q_h(s, a) \triangleq (1 - \gamma_h)\, h(s) + \gamma_h\, \max\!\left\{h(s),\, V_h(F(s, a))\right\},
\qquad
V_h(s) = \min_{a \in \calA} Q_h(s, a),
\end{equation}
is a contraction whose fixed point $Q_{h, \gamma_h}^\star$ approaches the undiscounted optimum $\Qh^\star$ as $\gamma_h \to 1$~\citep{fisac2019bridging, so2024solving, sharpless2026dualobjective}.

With $V_h^\star$ characterized, we reformulate Problem~\ref{prob:hard-constrained} in feasibility-dependent form by combining Definitions~\ref{def:reachability-values}--\ref{def:feasible-region} with an objective on $\calS \setminus \calS_f^\star$:

\begin{problem}[Reachability-value reformulation]
\label{prob:reachability}
\begin{align}
\max_\pi \quad & \E_{s_0 \sim \rho_0}\!\left[V_r^\pi(s_0)\, \indicator{s_0 \in \calS_f^\star} \;-\; V_h^\pi(s_0)\, \indicator{s_0 \notin \calS_f^\star}\right] \label{eq:rcrl} \\
\text{s.t.} \quad & V_h^\pi(s_0) \leq 0, \quad \rho_0\text{-a.s.\ }s_0 \in \calS_f^\star. \nonumber
\end{align}
\end{problem}
\noindent On feasible initial states, the equivalence in Definition~\ref{def:reachability-values} makes Problem~\ref{prob:reachability} equivalent to Problem~\ref{prob:hard-constrained}; on infeasible initial states, the objective switches to minimizing the worst-case future violation instead of leaving the constraint vacuously unsatisfiable. This is the reachability-constrained formulation of RCRL~\citep{yu2022reachability}, which FISOR~\citep{zheng2024safe} also adopts in the offline setting.

\begin{figure}[t]
\centering
\includegraphics[width=\linewidth]{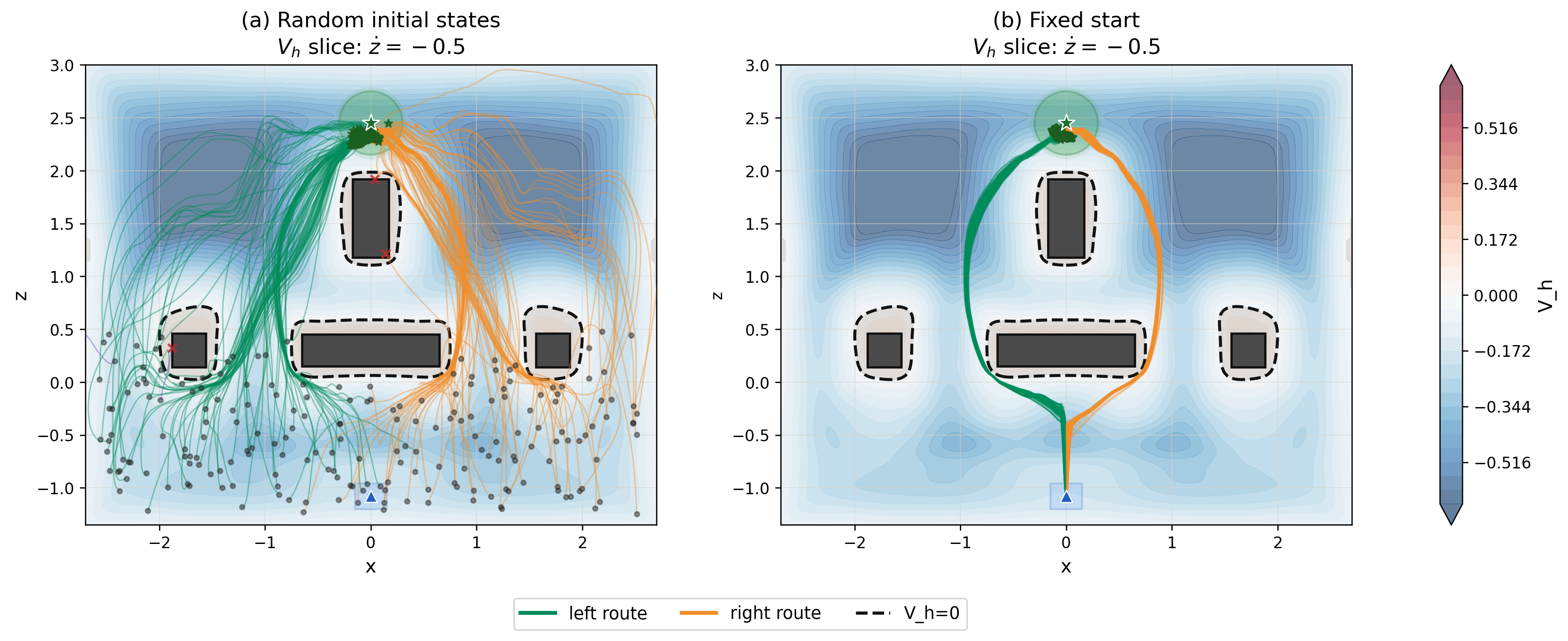}
\caption{\textbf{SSM on a planar-quadrotor stabilize-and-avoid task.} Trajectory colors indicate the two route modes learned by the policy rather than safe and unsafe outcomes: green trajectories pass to the left of the obstacles, and orange trajectories pass to the right. \emph{(a)}~Rollouts from randomly sampled initial states. \emph{(b)}~Rollouts from the nominal, symmetric initial state (blue triangle), perturbed within the small neighborhood shown by the light-blue box; from nearly the same initial state, the policy takes either route to the goal. The green star marks the goal center, and the surrounding green disk is the goal region, whose entry counts as a successful episode. Red crosses mark collisions. The background is a two-dimensional slice of the learned HJ value $\Vh$ at fixed $\dot z = -0.5$, and the dashed contour around the gray regions is its zero level set $\Vh = 0$, the estimated boundary of the viable set $\{s : \Vh(s) \le 0\}$.}
\label{fig:teaser}
\end{figure}

\section{Method: Safe Score Matching}
\label{sec:method}
\label{sec:theory_alg} %

The formulation above gives the desired hard-safety semantics, but it does not yet specify a tractable actor update for a diffusion policy. The action distribution that solves Problem~\ref{prob:reachability} at a given state can be highly complicated: $\calA_f^\star(s)$ is in general non-convex and may be disconnected, and the reward landscape over its interior may admit multiple separated high-value basins. A deterministic actor commits to one such component; a unimodal Gaussian averages across components or places mass in low-value or unsafe regions between them. On the other hand, diffusion policies~\citep{wang2022diffusion, chi2025diffusion} are expressive enough to represent complex distributions and thus become a natural fit for parametrizing the policies that solve Problem~\ref{prob:reachability}.

However, training a diffusion policy in online RL is non-trivial: a direct policy gradient through the multi-step reverse chain is expensive and high-variance. Q-Score Matching~\citep[QSM,][]{psenka2023learning} avoids this by exploiting the identity between $\nabla_a Q_r$ and the score field of the maximum-entropy reward-optimal policy: it regresses $\nabla_a Q_r$ as a time-independent target via a denoising loss, so the diffusion model learns this score directly without differentiating through the reverse chain. A primer on QSM and a discussion of its trade-offs are in Appendix~\ref{app:qsm-background}.

The QSM target, however, encodes only reward maximization. Problem~\ref{prob:reachability} adds two structural requirements that this target does not capture: a feasibility-dependent split between reward improvement on $\calS_f^\star$ and worst-case-violation descent on its complement, and a support restriction to HJ-viable actions in the feasible branch. SSM accommodates both inside the score-regression framework. Holding the critics fixed, we translate the reachability split into a one-step, action-level policy-improvement (PI) surrogate whose feasible branch encodes the support restriction (Section~\ref{sec:derivation}). Entropy-regularizing this surrogate yields a closed-form Boltzmann target, and SSM regresses its support-interior score with a QSM-style denoising loss (Section~\ref{sec:score}).

\subsection{HJ-Gated Local Policy Improvement}
\label{sec:derivation}
\label{sec:theory}             %
\label{sec:hj}                 %
\label{sec:action-feasibility} %
\label{sec:gibbs}              %

\paragraph{Local policy-improvement surrogate.}
Fix a state $s$. The state gate $s \in \calS_f^\star$ determines which objective is active, and the viable-action set $\calA_f^\star(s)$ determines which actions are admissible in the feasible branch. With oracle critics, these yield
\begin{align}
s \in \calS_f^\star :\quad \max_{\mu(\cdot \mid s)} \;\; & \E_{a \sim \mu}\!\big[\Qr^\pi(s, a)\big] \tag{PI-feasible}\label{eq:pi-feasible} \\
\text{s.t.} \quad & \operatorname{supp}\mu(\cdot \mid s) \subseteq \calA_f^\star(s), \nonumber \\[2pt]
s \notin \calS_f^\star :\quad \max_{\mu(\cdot \mid s)} \;\; & \E_{a \sim \mu}\!\big[-\Qh^\star(s, a)\big]. \tag{PI-infeasible}\label{eq:pi-infeasible}
\end{align}
The feasible subproblem is reward improvement restricted to actions that preserve the oracle feasible set. The infeasible subproblem is not a recovery guarantee: when $s \notin \calS_f^\star$, hard safety is generally unenforceable by definition, so the local objective biases actions toward smaller worst-case violation. We therefore interpret~\eqref{eq:pi-feasible}--\eqref{eq:pi-infeasible} as fixed-critic actor targets rather than an exact decomposition of the trajectory-level control problem. The split mirrors the reachability-dependent objective used in RCRL and FISOR~\citep{yu2022reachability,zheng2024safe}, but our use differs: FISOR is an offline method that extracts a diffusion policy through feasibility-dependent weighted behavior cloning of a fixed dataset, whereas SSM trains an online diffusion actor by regressing the score field derived below.

\paragraph{Entropy-regularized target density.}
To obtain a density target amenable to score regression, we consider the entropy-regularized form of~\eqref{eq:pi-feasible}--\eqref{eq:pi-infeasible}: in feasible states we maximize $\E[\alpha_r\Qr^\pi(s,a)]+\Hent(\mu(\cdot\mid s))$ over densities supported on $\calA_f^\star(s)$, and in infeasible states we maximize $\E[-\beta\Qh^\star(s,a)]+\Hent(\mu(\cdot\mid s))$ over $\calA$, with inverse temperatures $\alpha_r,\beta>0$. Theorem~\ref{thm:gibbs-two-end} gives the closed-form target. Theorems~\ref{thm:gibbs-two-end} and~\ref{thm:two-end-score} hold for fixed critics: they specify the target that a single actor update should match, and Section~\ref{sec:practical-algorithm} approximates this target with learned critics and a finite reverse chain.
\begin{theorem}[HJ-Gated Boltzmann Target]
\label{thm:gibbs-two-end}
The entropy-regularized local subproblem has the unique maximizer
\begin{equation}
\label{eq:pi-gibbs}
\pi^\star(a \mid s) =
\begin{cases}
\dfrac{\exp(\alpha_r\, \Qr^\pi(s, a))\, \indicator{\Qh^\star(s, a) \leq 0}}{Z_f(s)}, & s \in \calS_f^\star, \\[10pt]
\dfrac{\exp(-\beta\, \Qh^\star(s, a))}{Z_{\mathrm{rec}}(s)}, & s \notin \calS_f^\star,
\end{cases}
\end{equation}
where $Z_f(s)$ and $Z_{\mathrm{rec}}(s)$ are the corresponding partition functions.
\end{theorem}

We prove Theorem~\ref{thm:gibbs-two-end} in
Appendix~\ref{app:proof-gibbs} via the Gibbs /
Donsker--Varadhan variational identity. Differentiating the
log-density in \eqref{eq:pi-gibbs} yields the
score used by SSM.

\begin{theorem}[Score of the HJ-Gated Boltzmann Target]
\label{thm:two-end-score}
The score of $\pi^\star$ in~\eqref{eq:pi-gibbs} is
\begin{equation}
\label{eq:interior-score}
\nabla_a \log \pi^\star(a \mid s) =
\begin{cases}
\alpha_r\, \nabla_a \Qr^\pi(s, a), & s \in \calS_f^\star,\; \Qh^\star(s, a) < 0, \\[3pt]
-\beta\, \nabla_a \Qh^\star(s, a), & s \notin \calS_f^\star.
\end{cases}
\end{equation}
\end{theorem}

\noindent We prove this in Appendix~\ref{app:proof-two-end}, where Proposition~\ref{prop:zero-temp} also gives a complementary zero-temperature concentration result. On a viable state, the target is the maximum-entropy reward target of QSM restricted to the actions that the safety critic admits. On an infeasible state, it replaces the reward by the worst-case violation and favors actions predicted to reduce it. In the unconstrained limit $\calS_f^\star = \calS$ and $\calA_f^\star(s) = \calA$ the gate is inactive, and~\eqref{eq:interior-score} reduces to the QSM reward score $\alpha_r \nabla_a \Qr^\pi(s, a)$~\citep{psenka2023learning}.

The closed-form target~\eqref{eq:pi-gibbs} also makes the multimodality of the feasible branch concrete: its modes may arise from peaks of $\Qr^\pi(s, \cdot)$, from disconnected components of $\calA_f^\star(s)$, or from their intersection. An energy-based sampler matches the same target but requires iterative MCMC; a diffusion denoiser trained by score regression instead preserves the closed-form, the entropy regularization, and finite-step sampling at once.

\subsection{SSM Denoising Target}
\label{sec:score} %

\paragraph{Diffusion parameterization.}
We parameterize the actor with a finite-step diffusion sampler over actions~\citep{ho2020denoising}. Given a clean action $a_0$ from the replay buffer and a variance-preserving schedule $\{\beta_t\}_{t = 1}^T$, unrelated to the recovery coefficient $\beta$ of Theorem~\ref{thm:gibbs-two-end}, with $\alpha_t = 1 - \beta_t$, $\bar\alpha_t = \prod_{i = 1}^t \alpha_i$, and $\sigma(t) = \sqrt{1 - \bar\alpha_t}$, the forward noising process is
\begin{equation}
\label{eq:forward}
a_t = \sqrt{\bar\alpha_t}\, a_0 + \sigma(t)\, \epsilon, \qquad \epsilon \sim \mathcal{N}(0, I).
\end{equation}
At evaluation time, the policy samples from $a_T \sim \mathcal{N}(0, I)$ by the projected reverse update
\begin{equation}
\label{eq:ddpm_reverse}
a_{t - 1} = \Pi_\calA\!\!\left(\tfrac{1}{\sqrt{\alpha_t}}\!\left(a_t - \tfrac{\beta_t}{\sqrt{1 - \bar\alpha_t}}\, \epsilon_\theta(s, a_t, t)\right) + \tilde\sigma_t\, z\right),
\qquad z \sim \mathcal{N}(0, I),
\end{equation}
where $\tilde\sigma_t = \sqrt{\beta_t}$ for $t > 1$, $\tilde\sigma_1 = 0$, and $\Pi_\calA$ is coordinate-wise projection onto $\calA$. We retain the symbol $\epsilon_\theta$ to align with QSM~\citep{psenka2023learning}: the regression target below is a time-independent function of $(s, a_t)$ rather than the forward noise $\epsilon$, so $\epsilon_\theta$ functions at sampling time as a learned guidance vector field plugged into~\eqref{eq:ddpm_reverse}, not as a strict predictor of $\epsilon$.

\paragraph{HJ-gated guidance field.}
The clean-action score from Theorem~\ref{thm:two-end-score} defines the guidance field
\begin{equation}
\label{eq:score_target}
\bar\varphi(s, a) =
\begin{cases}
\alpha_r\, \gqr(s, a), & \Vh(s) \leq 0,\; \Qh(s, a) \leq 0 \quad \text{(reward branch)}, \\[3pt]
-\beta\, \gqh(s, a), & \Vh(s) > 0 \quad \text{(recovery branch)},
\end{cases}
\end{equation}
with $\bar\varphi(s, a) = 0$ otherwise. The zero case corresponds to actions classified as nonviable while the state is feasible; the target density has zero support there, so the classical score is not defined. Appendix~\ref{app:impl-algo} describes the guidance that our implementation assigns on this set and the normalization of the reward branch.

SSM trains the reverse-chain vector field by the squared regression loss
\begin{align}
\epsilon_{\mathrm{target}}(s, a_t)
&= -M_q\, \bar\varphi(s, a_t),
\label{eq:eps-target}\\
\calL_{\mathrm{SSM}}(\theta)
&= \E_{(s, a_0) \sim \calB,\, t,\, \epsilon}\!\left[\big\|\epsilon_\theta(s, a_t, t) - \epsilon_{\mathrm{target}}(s, a_t)\big\|_2^2\right],
\label{eq:loss}
\end{align}
where $M_q > 0$ is a guidance-strength coefficient, $t \sim \mathrm{Unif}\{1, \ldots, T\}$, $a_t$ is generated by~\eqref{eq:forward}, and gradients do not propagate through $\bar\varphi$. At the population optimum, the learned vector field equals $-M_q \bar\varphi$ under the training distribution; substituting this into~\eqref{eq:ddpm_reverse} adds a positive-coefficient component in the direction of $\bar\varphi$, hence in the direction of the HJ-gated local policy-improvement target.

\section{Practical Algorithm}
\label{sec:practical-algorithm}
\label{sec:learning-objectives} %

Section~\ref{sec:method} derived the score regression target~\eqref{eq:score_target} and loss~\eqref{eq:loss} for given critics $(\Vh, \Qh, \Qr)$. In the online algorithm these critics are learned from replay by off-policy temporal-difference regression and substituted into~\eqref{eq:score_target}. The reward-critic and denoiser updates follow standard practice and are deferred, together with the full pseudocode, to Appendix~\ref{app:impl-algo}. This section describes the two implementation choices specific to SSM: training the HJ safety critic and approximating the state gate that routes the two score branches.

\paragraph{Safety critic and HJ gate.}
We train $\Qh$ by mean-squared regression toward the discounted HJ Bellman target~\eqref{eq:hj_bellman}, using an EMA target network for the bootstrap. The Bellman target requires the inner minimum $\Vh(s) = \min_{a \in \calA} \Qh(s, a)$, which we estimate by finite-candidate minimization,
\begin{equation}
\label{eq:vh-mc}
\widehat \Vh(s) \;=\; \min_{a \in \calC_K(s)} \Qh(s, a),
\end{equation}
where $\calC_K(s)$ contains a single sample $a_\theta \sim \pi_\theta(\cdot \mid s)$ and $K-1$ Gaussian perturbations of $a_\theta$, for $K$ candidates per evaluation. Since $\calC_K(s) \subseteq \calA$, $\widehat \Vh(s) \ge \Vh(s)$ at fixed critic (Proposition~\ref{prop:conservative-gate}); in the online loop the bootstrap itself uses $\widehat \Vh$, so this conservatism is a gate-level property and does not extend to the learned critic. We use $\widehat \Vh(s) \le 0$ as the feasibility indicator in~\eqref{eq:score_target}. Candidate construction and an expectile-smoothed alternative used on Quad2D are in Appendices~\ref{app:vh-estimator} and~\ref{app:expectile}.

\paragraph{Learned-critic error.}
A bound on the error of the learned critic becomes a margin for both gates.

\begin{proposition}[Gate margin under bounded critic error]
\label{prop:gate-margin}
Let $V_h^\star(s) = \min_{a \in \calA} \Qh^\star(s, a)$, and suppose that $|\Qh(s, a) - \Qh^\star(s, a)| \le \varepsilon_Q$ for every state--action pair queried by the gates. Then
\begin{equation}
\label{eq:gate-margin}
\widehat \Vh(s) \le -\varepsilon_Q \;\Longrightarrow\; V_h^\star(s) \le 0,
\qquad
\Qh(s, a) \le -\varepsilon_Q \;\Longrightarrow\; \Qh^\star(s, a) \le 0.
\end{equation}
With zero thresholds, a state or an action is classified as feasible in error only if $0 < V_h^\star(s) \le \varepsilon_Q$ or $0 < \Qh^\star(s, a) \le \varepsilon_Q$, respectively.
\end{proposition}

Appendix~\ref{app:gate-margin} gives the proof. Replacing both zero thresholds by $-\varepsilon_Q$ therefore keeps an optimistic critic error of size $\varepsilon_Q$ from admitting an unsafe state or action to the reward branch, as in constraint tightening with approximate value functions~\citep{chatzikiriakos2024softmpc}. Our experiments keep zero thresholds (safety margin $\delta = 0$ in Table~\ref{tab:hyperparams-shared}), so under the bound such errors are confined to states and actions whose exact safety value lies in $(0, \varepsilon_Q]$. The proposition assumes $\varepsilon_Q$ rather than establishing it, and Section~\ref{sec:discussion} discusses why such a bound is hard to obtain.

\paragraph{Candidate-set ablation.}
We ablate the candidate set~\eqref{eq:candidate-set} on Quad3D. The task metric is the \emph{terminal tracking $\ell_1$}, the mean $\ell_1$ distance $\|s_T - s_{\mathrm{ref}}\|_1$ between the final state and the target over evaluation episodes, including those that end in a crash. Local Gaussian proposals are essential: replacing them with independent reverse-chain samples substantially increases the terminal tracking $\ell_1$, because samples from $\pi_\theta$ cluster around its modes and miss low-$\Qh$ actions between them. The proposal scale has an interior optimum at $\sigma_\eta = 0.3$. Increasing $K$ from $8$ to $32$ at $\sigma_\eta = 0.3$ eliminates the observed episode violations at a higher terminal tracking $\ell_1$, so $K$ acts as a task--safety operating knob. Appendix~\ref{app:ablation-quad3d} reports the full $K \times \sigma_\eta$ grid and an expectile-$\Vh^\tau$ control (Table~\ref{tab:ablation-full}).

\section{Experiments}
\label{sec:experiments}

We ask whether SSM improves task performance over primal--dual and reachability-based baselines without accepting more unsafe behavior, and whether this holds on a standard safe RL benchmark.

\paragraph{Benchmarks.}
The main suite contains three safety-critical control tasks of increasing complexity (Figure~\ref{fig:benchmarks}). Quad2D tracks a circular reference inside a narrow altitude band, Quad3D regulates a nine-state quadrotor to a target state, and F16 stabilizes a 16-state fixed-wing aircraft in a low-altitude band within flight-envelope and position bounds. We also evaluate on the HalfCheetah and Swimmer velocity tasks of Safety-Gymnasium~\citep{ji2023safety} and use its CarGoal1 navigation task for two further studies. Appendix~\ref{app:dynamics} specifies the control tasks, and Appendix~\ref{app:safety-gym} the Safety-Gymnasium setup.

\begin{figure}[t]
    \centering
    \includegraphics[width=0.82\linewidth]{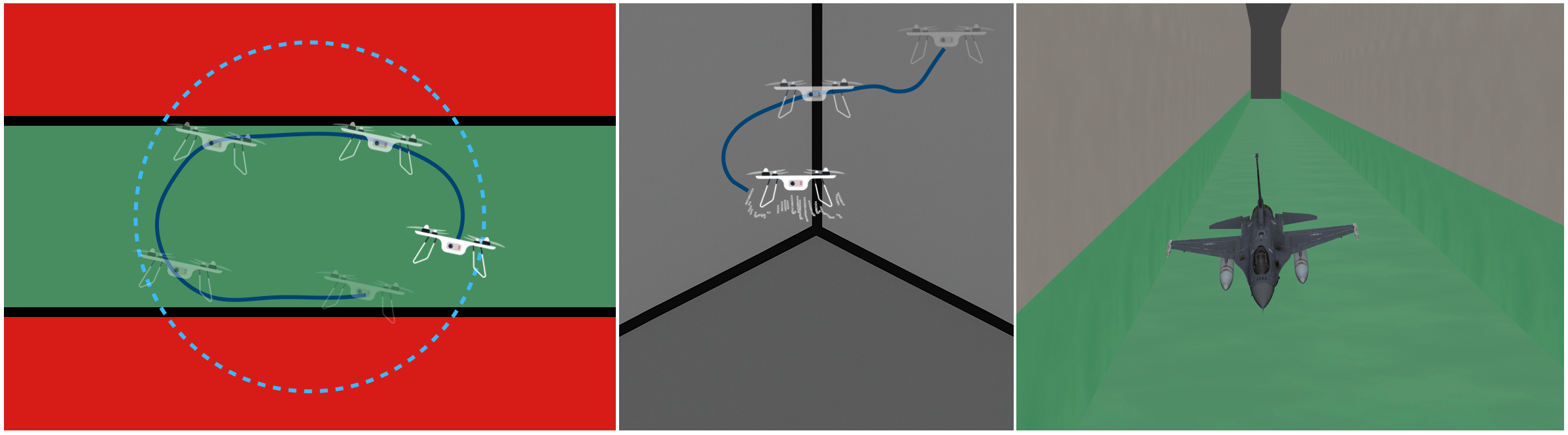}
    \caption{\textbf{Benchmark environments.} Renderings of the Quad2D, Quad3D, and F16 tasks.}
    \label{fig:benchmarks}
\end{figure}

\paragraph{Baselines.}
CAL~\citep{wu2024offpolicy} is an off-policy primal--dual method that enforces an expected-cost budget with a conservative augmented Lagrangian; on Safety-Gymnasium we add ALGD~\citep{algd2026}, a primal--dual method with a diffusion actor. RAC~\citep{yu2022reachability} and RESPO~\citep{ganai2023iterative} instead optimize state-wise reachability objectives; RAC is off-policy, and RESPO is on-policy with substantially more environment steps. EFPPO~\citep{so2023solving} solves the stabilize-avoid problem in epigraph form and is evaluated on F16, the task on which it was introduced. Table~\ref{tab:taxonomy} summarizes these properties. We rerun every baseline in our setup. On the control tasks, each baseline is trained for its recommended number of iterations or for at least three hours, and every method is evaluated at its best checkpoint. On Safety-Gymnasium, each method is trained with five seeds and evaluated at a fixed endpoint (Appendix~\ref{app:safety-gym}).

\paragraph{Metrics.}
The task metrics are the Quad2D tracking error, the Quad3D terminal tracking $\ell_1$ (Section~\ref{sec:practical-algorithm}), and the F16 stabilization rate, the fraction of episodes that hold the target altitude band for 50 consecutive steps. On Quad3D we also report the final-50-step tracking $\ell_1$, which averages the same distance over the last 50 steps and thus measures whether the vehicle stays stabilized. For safety, each method classifies every evaluated initial state as safe or unsafe with its learned safety estimate. A rollout of its policy then decides whether the state is actually safe, that is, whether $h(s_t) \le 0$ at every step. With safe as the positive class, TP counts safe states predicted safe, FP unsafe states predicted safe, TN unsafe states predicted unsafe, and FN safe states predicted unsafe. The safety precision, safe coverage, and false-safe rate (FSR) are
\begin{equation}
\label{eq:safety-metrics}
\text{Precision} = \frac{TP}{TP + FP},
\quad
\text{Coverage} = \frac{TP + FP}{TP + TN + FP + FN},
\quad
\text{FSR} = \frac{FP}{TP + FP}.
\end{equation}
The three metrics are read jointly: a method can lower its false-safe rate by predicting fewer states safe, which lowers coverage, and it can raise coverage by admitting more unsafe states, which lowers precision. Appendix~\ref{app:eval-protocol} details the evaluation protocol.

\begin{figure}[t]
    \centering
    \includegraphics[width=\linewidth]{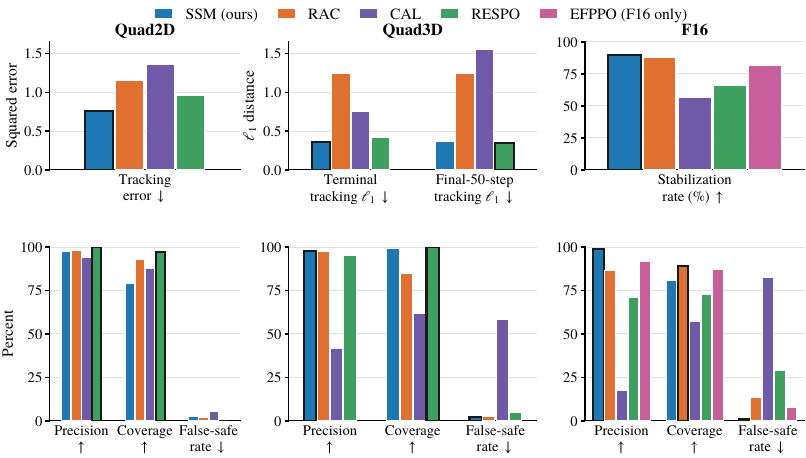}
    \caption{\textbf{Task performance (top) and safety prediction (bottom).} Arrows mark the preferred direction, and a dark outline marks the best method for each metric. Precision and the false-safe rate are computed over the initial states that a method predicts safe, and coverage is the fraction of states it predicts safe~\eqref{eq:safety-metrics}.}
    \label{fig:metrics}
\end{figure}

\paragraph{Results on the control benchmarks.}
SSM attains the lowest Quad2D tracking error, the lowest Quad3D terminal tracking $\ell_1$, and the highest F16 stabilization rate (Figure~\ref{fig:metrics}), while RESPO attains a slightly lower Quad3D final-50-step tracking $\ell_1$. On Quad3D and F16, SSM combines this task performance with the highest precision. On F16, RAC and EFPPO, the baselines closest to SSM in stabilization rate, accept more unsafe states. CAL, which constrains the cost only in expectation, accepts the most unsafe states on every task. Quad2D, whose reference repeatedly approaches the constraint boundary, is the exception: SSM predicts the fewest states safe, and RESPO attains higher precision and coverage. This pattern is consistent with the design of SSM. The gate restricts reward guidance to states and actions that the safety critic classifies as viable, and the diffusion actor keeps several viable routes (Figure~\ref{fig:teaser}) instead of committing to one.

\begin{table}[t]
\caption{\textbf{Safety-Gymnasium velocity tasks.} Final episodic reward and cost (mean $\pm$ SD over five seeds). The cost counts the steps whose velocity exceeds the task threshold; a method is within budget on a task when its mean cost over the five seeds is at most $d = 25$, the budget that CAL and ALGD train against. RESPO trains for 10M environment steps and the others for 1M. Red: mean cost above budget. Among entries within budget, bold blue and bold green mark the highest and second-highest reward. SSM uses the posterior-target implementation of Appendix~\ref{app:safety-gym}.}
\label{tab:safety-gym}
\centering
\footnotesize
\renewcommand{\arraystretch}{0.95}
\setlength{\tabcolsep}{4pt}
\begin{tabular}{llccccc}
\toprule
Task & Metric & RESPO & CAL & RAC & ALGD & SSM (ours) \\
\midrule
\multirow{2}{*}{Swimmer} & Reward & $37\pm1$ & \ob{31\pm7} & \bestR{75\pm57} & \secondR{49\pm4} & $45\pm6$ \\
 & Cost & $8.4\pm1.3$ & \ob{27.9\pm21.3} & $14.5\pm28.4$ & $6.5\pm2.4$ & $0.5\pm0.4$ \\
\specialrule{0.3pt}{0.6pt}{0.6pt}
\multirow{2}{*}{HalfCheetah} & Reward & \secondR{2323\pm323} & \ob{2631\pm107} & \ob{5014\pm3781} & \ob{2695\pm99} & \bestR{2754\pm13} \\
 & Cost & $8.4\pm3.9$ & \ob{28.3\pm12.5} & \ob{391.7\pm536.3} & \ob{42.0\pm39.7} & $0.0\pm0.1$ \\
\midrule
\multicolumn{2}{l}{Tasks within budget} & 2/2 & 0/2 & 1/2 & 1/2 & 2/2 \\
\bottomrule
\end{tabular}
\end{table}

\paragraph{Safety-Gymnasium.}
SSM has the lowest mean cost on both velocity tasks and the smallest spread in cost across seeds (Table~\ref{tab:safety-gym}). On HalfCheetah, it attains the highest reward among the entries within budget, whereas CAL, ALGD, and RAC exceed the budget. Two of the five RAC seeds violate the constraint at almost every step (Appendix~\ref{app:safety-gym}). On Swimmer, SSM keeps its cost at 0.5, more than ten times lower than that of any other method, while RAC and ALGD attain higher reward within budget and RAC has by far the largest spread in reward across seeds. This near-zero cost is consistent with the state-wise objective of SSM, which targets $h(s_t) \le 0$ at every step rather than an expected cost below the budget that CAL and ALGD train against. The cost of SSM peaks within the first 100k environment steps and remains low afterwards (Figure~\ref{fig:velocity-curves}).

\paragraph{Additional studies.}
On CarGoal1, varying $\alpha_r$, $\beta$, or $M_q$ from $0.25\times$ to $4\times$ its default changes the reward and the fraction of violation-free episodes only slightly (Appendix~\ref{app:safety-gym}). Under the same expected-cost formulation, a diffusion actor trained by QSM attains higher reward and lower cost than a Gaussian SAC actor, so the policy class affects the reward--cost trade-off even without HJ routing.

\paragraph{Computation.}
On Quad3D, SSM reaches its reported performance with 4 to 6 times fewer environment steps than RAC and RESPO. It trains in about one hour, compared with about 2.5 hours for RESPO, the fastest baseline in wall-clock time. At deployment, SSM runs one five-step reverse chain per action.

\section{Conclusion and Limitations}
\label{sec:discussion}
\label{sec:discussion-limitations} %

We presented Safe Score Matching (SSM), an online actor--critic method that trains a diffusion policy under a state-wise HJ reachability formulation. Its actor target sends reward guidance to states and actions that the safety critic classifies as viable and recovery guidance to infeasible states. A QSM-style denoising loss regresses this target without backpropagating through the reverse chain. On quadrotor and fixed-wing benchmarks, SSM attains the best or near-best task metric with low false-safe rates, and on Safety-Gymnasium velocity tasks it attains the lowest cost with competitive reward.

The main limitation is the reliance of SSM on a learned safety critic. Theorems~\ref{thm:gibbs-two-end}--\ref{thm:two-end-score} characterize the actor target for fixed critics, and Proposition~\ref{prop:gate-margin} turns a critic-error bound $\varepsilon_Q$ into a gate margin without establishing that bound. Learning an accurate HJ safety critic online from nonstationary replay data remains an open challenge for safe RL, and it is hard to know in general when such a bound holds. Progress on this problem would directly strengthen methods such as SSM. The analysis also assumes deterministic dynamics. Extending the safety value to probabilistic or robust (HJ--Isaacs) reachability, and validating the resulting policies on real robots, are important next steps.\label{endofmaintext}

\begin{ack}
We thank William Sharpless, Xiaoyuan Cheng, and Wenxuan Yuan for their thoughtful and inspiring discussions and insights.
This work is supported by ONR YIP (\#N00014-22-1-2292) and the UCSD JSOE Early Career Faculty Award.
The authors declare no competing interests.
\end{ack}

\bibliographystyle{plainnat}
\bibliography{references}

\appendix

\providecommand{\KL}{\mathrm{KL}}

\section{Extended Related Work}
\label{app:related}
\paragraph{Method positioning at a glance.}
Table~\ref{tab:positioning} summarizes how SSM relates to existing safe RL families along four dimensions: online vs.\ offline training, the safety semantics enforced, the policy class, and the actor update rule.

\paragraph{Expected-cost safe RL.}
A large body of online safe RL formulates safety as a CMDP with an expected cumulative-cost budget and solves it by primal--dual or Lagrangian updates: CPO~\citep{achiam2017constrained}, RCPO~\citep{tessler2018reward}, FOCOPS~\citep{zhang2020first}, SAUTE~\citep{sootla2022saute}, PPO-Lag and SAC-Lag~\citep{ray2019benchmarking}, PID-Lagrangian~\citep{stooke2020responsive}, CUP~\citep{yang2022constrained}, CVPO~\citep{liu2022constrained}, CCAC~\citep{liu2023towards}, and off-policy CAL~\citep{wu2024offpolicy}. Although these methods enjoy convergence guarantees under mild assumptions~\citep{altman2021constrained}, the constraint is enforced only on average and the primal--dual loop is unstable in practice~\citep{so2023solving, zhang2025discrete}. Most of these methods are on-policy by design; off-policy variants such as SAC-Lag, CVPO, and CAL exist but are particularly susceptible to value-estimation error, where distributional shift compounds critic error and induces oscillating dual variables~\citep{wu2024offpolicy}.

\paragraph{State-wise / reachability-based safe RL.}
A second line of methods enforces state-wise zero-violation by constraining the policy to a control-invariant set: Lyapunov-based methods~\citep{chow2018lyapunov, zhang2024learning}, control barrier functions~\citep{qin2022sablas, ma2021model, so2024train}, safety shielding~\citep{wagener2021safe}, projection-based filters~\citep{lin2024projection}, and Hamilton--Jacobi (HJ) reachability analysis~\citep{fisac2019bridging, hsu2021safety, ganai2024hamilton, choi2025data}. Within HJ-based safe RL, FAC~\citep{ma2021feasible} introduces a state-wise Lagrange multiplier with complementary slackness; RCRL~\citep{yu2022reachability} uses HJ self-consistency to characterize the largest feasible set; RESPO~\citep{ganai2023iterative} formulates a feasibility-dependent objective in stochastic settings; FRPI~\citep{qin2024feasible} provides a feasible reachable policy iteration; DRAC~\citep{li2024dual} poses safe RL under adversarial disturbances as a constrained two-player zero-sum Markov game whose optimal safety value identifies the maximal robust invariant set; and DOHJ-PPO~\citep{sharpless2026dualobjective} extends HJ-based PPO to reach-always-avoid and reach-reach problems via novel decomposed Bellman formulations. All of these use Gaussian or deterministic actors, which are limited in representing non-convex feasible-action distributions. SSM inherits the feasible/infeasible decomposition from this line but replaces the policy class with a diffusion model and the update rule with HJ-gated score matching.

\paragraph{Diffusion / flow policies in RL.}
Diffusion policies were first developed for offline RL, where the policy is fitted to a fixed dataset~\citep{janner2022planning, wang2022diffusion, chi2025diffusion, kang2023efficient, ren2024diffusion, park2025flow}, with extensions for entropy regularization~\citep{wang2024diffusion, dong2025maximum} and flexible target sampling~\citep{ding2024diffusion, uehara2024understanding}. Online RL with diffusion policies is more recent: QSM~\citep{psenka2023learning} links the diffusion-policy score to the action-gradient of $\Qr$ and trains the denoiser on a single modified target; SDAC~\citep{ma2025soft}, FlowRL~\citep{lv2025flow}, SAC Flow~\citep{zhang2026sac}, and FPO~\citep{mcallister2026flow} explore related score- and flow-matching objectives, and RFM~\citep{li2026reverse} casts the training target as a posterior mean given an intermediate noisy action, which unifies noise-expectation and gradient-expectation estimators and extends Boltzmann-target training from diffusion to flow policies. All are reward-only. SSM extends QSM's score-target machinery to a piecewise reward / recovery target indexed by HJ feasibility.

\paragraph{Safe generative policies.}
Closest to our setting are methods that combine generative policies with safety. SafeDiffuser~\citep{xiao2023safediffuser} performs constrained planning with a diffusion model; \citet{cheng2025safe} use Lyapunov-guided diffusion for stable control; \citet{zhang2025constrained} apply constrained diffusion sampling for safe planning. In the offline regime, FISOR~\citep{zheng2024safe} combines hard safety, HJ-style feasibility, and a guided diffusion policy via feasibility-dependent weighted behavior cloning extracted by a guided diffusion model. SSM studies the online actor--critic counterpart: it derives a denoising score target from the HJ avoid state-action value $\Qh$ and the viable-action set $\{a : \Qh(s, a) \le 0\}$, rather than from a soft cumulative-cost objective. ALGD~\citep{algd2026} trains a diffusion policy online under an expected cumulative-cost constraint; Section~\ref{sec:experiments} compares SSM with it on Safety-Gymnasium.

\section{Q-Score Matching: Background and Endpoint Guidance}
\label{app:qsm-background}

This appendix collects the QSM context underlying SSM's actor update, together with the trade-off between regressing the clean-action energy score and regressing exact intermediate-time scores. The discussion expands the brief inline reference in Section~\ref{sec:method}.

\paragraph{Score-based view of policy improvement.}
The score-based view of policy improvement underlying our method originates from Q-score matching (QSM)~\citep{psenka2023learning}. Treating the policy as a stochastic process whose actions evolve along a score field $\Psi(s, a)$, \citet{psenka2023learning} showed that the optimal $\Psi$ for reward maximization is, up to a positive constant, the action-gradient of the state-action value function: $\Psi^\star(s, a) \propto \nabla_a \Qr(s, a)$. The induced stationary action distribution is the Boltzmann form $\pi(a \mid s) \propto \exp(\alpha_r \Qr(s, a))$, recovering the soft-policy view that motivates entropy-regularized RL~\citep{song2020score}. Direct policy gradients for diffusion actors require backpropagating through the entire multi-step reverse sampling chain, which is computationally expensive and high-variance~\citep{psenka2023learning}; QSM bypasses this by regressing a $t$-independent score-matching target via a denoising loss, differentiating only through the denoiser at the current noise level. SSM inherits this design choice and extends QSM's reward-only target to a piecewise reward / recovery target indexed by HJ feasibility (Section~\ref{sec:method}, Theorems~\ref{thm:gibbs-two-end}--\ref{thm:two-end-score}).

\paragraph{Endpoint guidance versus exact intermediate diffusion score.}
The score in Theorem~\ref{thm:two-end-score} is the score of the clean-action target $\pi^\star(a \mid s)$. The exact score of the forward-smoothed intermediate marginals would instead be
\begin{equation}
\label{eq:exact-marginal-score}
\nabla_{a_t} \log p_t^\star(a_t \mid s),
\qquad
p_t^\star(a_t \mid s) = \int \pi^\star(a_0 \mid s)\, q_t(a_t \mid a_0)\, da_0,
\end{equation}
which is generally intractable in the hard-constrained setting because the indicator $\indicator{\Qh^\star(s, a) \le 0}$ inside $\pi^\star$ in~\eqref{eq:pi-gibbs} makes $p_t^\star$ a Gaussian convolution of a discontinuous density without a closed-form gradient. Following~\citet{psenka2023learning}, SSM regresses a $t$-independent target derived from the clean-action density. Since $-M_q\bar\varphi(s, a_t)$ depends only on $(s, a_t)$, the population $L^2$-minimizer of~\eqref{eq:loss} is $-M_q\bar\varphi$ itself; substituting it into the reverse update~\eqref{eq:ddpm_reverse} adds a positive-coefficient component aligned with~\eqref{eq:interior-score} on the support interior. The relation is therefore guidance-component alignment with the clean-energy score, not exact reverse-chain score matching for $p_t^\star$. The velocity-task implementation of Appendix~\ref{app:safety-gym} instead regresses a Monte Carlo estimate of the posterior-mean noise, averaged with the prediction of a lagged copy of the denoiser.

\paragraph{Connection to Langevin-style sampling.}
\citet{ma2025soft} characterize QSM as a Langevin-based sampler (using the analytic Q-gradient as score, without noise perturbation across denoising levels) rather than a multi-level noisy-score diffusion sampler. SSM inherits this characterization: the reverse update~\eqref{eq:ddpm_reverse} with the substituted denoiser is a diffusion-parameterized Langevin-style sampler on the HJ-gated energy landscape, not an exact multi-level noisy-score sampler. We therefore treat $M_q$ as a guidance-strength hyperparameter rather than assigning a closed-form stationary density to the finite-step chain.

\section{Finite-Candidate HJ Gate}
\label{app:hj}
\label{app:vh-estimator}

This appendix records the formal property of the finite-candidate gate~\eqref{eq:vh-mc} used in Section~\ref{sec:practical-algorithm}.

\paragraph{Candidate construction.}
For state $s$, let $\calC_K(s)$ denote the candidate action set
\begin{equation}
\label{eq:candidate-set}
\calC_K(s) \;=\; \{a_\theta\} \;\cup\; \big\{\Pi_\calA\!\big(a_\theta + \sigma_\eta\, \eta^{(j)}\big)\big\}_{j=1}^{K-1},
\qquad a_\theta \sim \pi_\theta(\cdot \mid s),
\quad \eta^{(j)} \sim \mathcal N(0, I).
\end{equation}
A single diffusion sample $a_\theta$ anchors the candidate set, and $K-1$ local Gaussian-proposal candidates surround it, for $K$ candidates per gate evaluation. The same construction is used for both the target-state gate $\widehat{\bar \Vh}(s_i')$ in the safety bootstrap and the current-state gate $\widehat{\Vh}(s_i)$ in the denoising update; the gates differ only in which safety critic is queried (EMA target $\bar\Qh$ versus online $\Qh$). Coupling all $K-1$ proposals to a single policy sample requires only one reverse-chain forward pass per gate evaluation, while the Gaussian perturbations introduce local action-space coverage at negligible additional cost.

\paragraph{Conservatism of the gate.}
The estimator is upward-biased relative to the continuous-action minimum, which translates into a conservative feasibility certificate at the gate level.

\begin{proposition}[Conservative finite-candidate HJ gate]
\label{prop:conservative-gate}
Fix a critic $\Qh(s, \cdot)$ and let $\Vh(s) = \min_{a \in \calA} \Qh(s, a)$ and $\widehat \Vh(s) = \min_{a \in \calC_K(s)} \Qh(s, a)$ with $\calC_K(s) \subseteq \calA$. Then $\widehat \Vh(s) \ge \Vh(s)$. Consequently, the finite-candidate feasible set $\widehat\calS_f \triangleq \{s : \widehat \Vh(s) \le 0\}$ is an inner approximation of the critic-induced feasible set $\calS_f \triangleq \{s : \Vh(s) \le 0\}$, i.e., $\widehat\calS_f \subseteq \calS_f$.
\end{proposition}

\begin{proof}
Minimizing over a subset cannot fall below minimizing over the superset, so $\widehat \Vh(s) \ge \Vh(s)$. If $s \in \widehat\calS_f$, then $\Vh(s) \le \widehat \Vh(s) \le 0$, so $s \in \calS_f$.
\end{proof}

Two caveats are worth stating explicitly. First, the proposition is stated for a fixed critic; in the online learning loop, the bootstrap target itself uses $\widehat \Vh$, so the learned $\Qh$ fits a finite-candidate Bellman fixed point rather than the oracle one, and learned-critic error can break the conservatism in either direction. Second, the proposition certifies state-level feasibility classification, not that every action sampled by the diffusion policy is viable; reward guidance is restricted separately to actions with $\Qh(s, a_t) \le 0$ by the indicator in the score target~\eqref{eq:score_target}. The score-matching framework of Section~\ref{sec:method} does not depend on the specific $\Vh$ estimator: alternatives such as adversarial safety actors~\citep{fisac2019bridging, hsu2021safety} or expectile regression~\citep{kostrikov2022offline, zheng2024safe} can be substituted, and Appendix~\ref{app:expectile} reports the variant we use on Quad2D.

\paragraph{Gate margin under bounded critic error.}
\label{app:gate-margin}
We prove Proposition~\ref{prop:gate-margin}.

\begin{proof}
Let $\hat a \in \calC_K(s)$ attain $\widehat\Vh(s)$ and suppose that $\widehat\Vh(s) \le -\varepsilon_Q$. Since $\calC_K(s) \subseteq \calA$, the error bound gives
\begin{equation*}
V_h^\star(s) \;\le\; \Qh^\star(s, \hat a) \;\le\; \Qh(s, \hat a) + \varepsilon_Q \;=\; \widehat\Vh(s) + \varepsilon_Q \;\le\; 0.
\end{equation*}
The same chain with $\hat a$ replaced by $a$ proves the action statement. With a zero threshold, $\widehat\Vh(s) \le 0$ implies $V_h^\star(s) \le \widehat\Vh(s) + \varepsilon_Q \le \varepsilon_Q$, so an erroneous feasible classification requires $0 < V_h^\star(s) \le \varepsilon_Q$; the action case is identical.
\end{proof}

\section{Proofs for Section~\ref{sec:method}}
\label{app:proofs}

This appendix proves Theorems~\ref{thm:gibbs-two-end} (HJ-Gated Boltzmann Target) and~\ref{thm:two-end-score} (Score of the HJ-Gated Boltzmann Target), together with the complementary zero-temperature concentration result Proposition~\ref{prop:zero-temp}. Theorem~\ref{thm:gibbs-two-end}'s assumptions are minimal: we only require measurability of the critics together with finiteness of the partition functions, which is the standard regularity needed for the Gibbs variational principle~\citep{boyd2004convex, brezis2011functional}. We separate this from the stronger differentiability and continuity hypotheses required in Theorem~\ref{thm:two-end-score} and Proposition~\ref{prop:zero-temp}, which are imposed only locally.

\subsection{Proof of Theorem~\ref{thm:gibbs-two-end}}
\label{app:proof-gibbs}

\paragraph{Entropy-regularized subproblems.}
We make the entropy-regularized analogues of~\eqref{eq:pi-feasible}--\eqref{eq:pi-infeasible} referenced in Section~\ref{sec:gibbs} explicit. Stating these explicitly with critics $(\Qr^\pi, \Qh^\star)$, with differential entropy $\Hent(\pi(\cdot \mid s)) \triangleq -\int \pi(a \mid s) \log \pi(a \mid s)\, da$ and inverse-temperature parameters $\alpha_r, \beta > 0$, the regularized feasible subproblem on $s \in \calS_f^\star$ is
\begin{align}
\max_{\pi(\cdot \mid s)} \quad & \E_{a \sim \pi}\!\big[\alpha_r\, \Qr^\pi(s, a)\big] + \Hent(\pi(\cdot \mid s)) \label{eq:gibbs-feasible} \\
\text{s.t.} \quad & \operatorname{supp}\pi(\cdot \mid s) \subseteq \calA_f^\star(s),
\quad
\textstyle\int \pi(a \mid s)\, da = 1, \nonumber
\end{align}
and the regularized recovery subproblem on $s \notin \calS_f^\star$ is
\begin{equation}
\label{eq:gibbs-infeasible}
\max_{\pi(\cdot \mid s)} \;\; \E_{a \sim \pi}\!\big[-\beta\, \Qh^\star(s, a)\big] + \Hent(\pi(\cdot \mid s)),
\qquad
\textstyle\int \pi(a \mid s)\, da = 1.
\end{equation}
The full assumptions of Theorem~\ref{thm:gibbs-two-end} are: $\calA \subset \R^d$ compact, $\Qr^\pi(s,\cdot)$ and $\Qh^\star(s,\cdot)$ Lebesgue-measurable on $\calA$, $\calA_f^\star(s)$ Lebesgue-measurable with $\lambda(\calA_f^\star(s)) > 0$ whenever $s \in \calS_f^\star$, and the partition functions
\begin{equation*}
Z_f(s) = \int_{\calA_f^\star(s)} \exp(\alpha_r\, \Qr^\pi(s, a))\, da,
\qquad
Z_{\mathrm{rec}}(s) = \int_\calA \exp(-\beta\, \Qh^\star(s, a))\, da
\end{equation*}
are finite. The maximization is over probability densities absolutely continuous with respect to Lebesgue measure, supported on $\calA_f^\star(s)$ (resp.\ $\calA$); uniqueness is up to Lebesgue-a.e.\ equivalence.

\paragraph{Donsker--Varadhan derivation.}
We give a KL-divergence derivation~\citep{boyd2004convex} of~\eqref{eq:pi-gibbs}, which makes the role of partition-function finiteness transparent and avoids differentiability of the critics. Throughout, $\pi(\cdot \mid s)$ denotes a candidate density with respect to Lebesgue measure on $\calA$, all integrals are over $\calA$ unless noted, and the state $s$ is held fixed.

\paragraph{Case 1: feasible state $s \in \calS_f^\star$.}
The feasible entropy-regularized subproblem~\eqref{eq:gibbs-feasible} reads, after expanding $\Hent$,
\begin{equation}
\label{eq:proof-feasible-primal}
\max_{\pi(\cdot \mid s)} \;\;
\int_{\calA_f^\star(s)} \pi(a \mid s)\, \alpha_r\, \Qr^\pi(s, a)\, da
\;-\; \int_{\calA_f^\star(s)} \pi(a \mid s)\, \log \pi(a \mid s)\, da
\end{equation}
over probability densities supported on $\calA_f^\star(s)$ and absolutely continuous with respect to Lebesgue measure. Define the candidate Boltzmann density
\begin{equation}
\label{eq:proof-feasible-candidate}
\pi^\star(a \mid s) \;\triangleq\; \frac{\exp\!\big(\alpha_r\, \Qr^\pi(s, a)\big)\, \indicator{a \in \calA_f^\star(s)}}{Z_f(s)},
\qquad
Z_f(s) \;\triangleq\; \int_{\calA_f^\star(s)} \exp(\alpha_r\, \Qr^\pi(s, a))\, da.
\end{equation}
Under the assumed measurability of $\calA_f^\star(s)$ and $\Qr^\pi(s, \cdot)$ together with $0 < \lambda(\calA_f^\star(s))$ (positive Lebesgue measure) and $Z_f(s) < \infty$, the candidate $\pi^\star$ is a well-defined probability density.

For any density $\pi$ supported on $\calA_f^\star(s)$, take logarithms in~\eqref{eq:proof-feasible-candidate} to obtain the pointwise identity $\alpha_r\, \Qr^\pi(s, a) = \log \pi^\star(a \mid s) + \log Z_f(s)$ on $\calA_f^\star(s)$. Substituting into the objective in~\eqref{eq:proof-feasible-primal},
\begin{align}
&\int \pi(a \mid s)\, \alpha_r\, \Qr^\pi(s, a)\, da \;-\; \int \pi(a \mid s)\, \log \pi(a \mid s)\, da \notag \\
&\qquad= \int \pi(a \mid s)\, \log\!\frac{\pi^\star(a \mid s)}{\pi(a \mid s)}\, da \;+\; \log Z_f(s) \notag \\
&\qquad= -\,\KL\!\big(\pi(\cdot \mid s) \,\big\|\, \pi^\star(\cdot \mid s)\big) \;+\; \log Z_f(s).
\label{eq:proof-feasible-kl}
\end{align}
The Donsker--Varadhan / Gibbs variational identity~\citep{boyd2004convex, brezis2011functional} states that $\KL(\pi \| \pi^\star) \geq 0$ with equality if and only if $\pi = \pi^\star$ a.e.\ on $\calA_f^\star(s)$. Hence the objective in~\eqref{eq:proof-feasible-primal} is uniformly upper-bounded by $\log Z_f(s)$, with the unique maximizer being $\pi^\star$ in~\eqref{eq:proof-feasible-candidate}. This is the feasible branch of~\eqref{eq:pi-gibbs}.

\paragraph{Case 2: infeasible state $s \notin \calS_f^\star$.}
The infeasible entropy-regularized subproblem~\eqref{eq:gibbs-infeasible} reads, after expanding $\Hent$,
\begin{equation}
\label{eq:proof-infeasible-primal}
\max_{\pi(\cdot \mid s)} \;\;
\int_\calA \pi(a \mid s)\, \big(-\beta\, \Qh^\star(s, a)\big)\, da
\;-\; \int_\calA \pi(a \mid s)\, \log \pi(a \mid s)\, da
\qquad \text{s.t.} \quad
\int_\calA \pi(a \mid s)\, da = 1,
\end{equation}
without a support constraint. Define the candidate Boltzmann density
\begin{equation}
\label{eq:proof-infeasible-candidate}
\pi^\star(a \mid s) \;\triangleq\; \frac{\exp\!\big(-\beta\, \Qh^\star(s, a)\big)}{Z_{\mathrm{rec}}(s)},
\qquad
Z_{\mathrm{rec}}(s) \;\triangleq\; \int_\calA \exp(-\beta\, \Qh^\star(s, a))\, da,
\end{equation}
which is a well-defined probability density under the assumed measurability of $\Qh^\star(s, \cdot)$ on $\calA$ and $Z_{\mathrm{rec}}(s) < \infty$. The same KL identity applied with $\pi^\star$ in place of the feasible candidate yields
\begin{equation*}
\int \pi (-\beta\, \Qh^\star)\, da - \int \pi \log \pi\, da \;=\; -\KL(\pi \,\|\, \pi^\star) + \log Z_{\mathrm{rec}}(s),
\end{equation*}
and Donsker--Varadhan again gives uniqueness: the unique maximizer is~\eqref{eq:proof-infeasible-candidate}, the infeasible-state branch of~\eqref{eq:pi-gibbs}.
\hfill$\square$

\subsection{Proof of Theorem~\ref{thm:two-end-score}}
\label{app:proof-two-end}

\begin{proof}[Proof of Theorem~\ref{thm:two-end-score} (Score of the HJ-Gated Boltzmann Target)]
Take logarithms of the closed-form $\pi^\star$ from Theorem~\ref{thm:gibbs-two-end} (proved in Appendix~\ref{app:proof-gibbs}). For $s \in \calS_f^\star$ and $a \in \mathrm{int}(\calA_f^\star(s))$,
\begin{equation*}
\log \pi^\star(a \mid s) \;=\; \alpha_r\, \Qr^\pi(s, a) \;-\; \log Z_f(s),
\end{equation*}
with $\log Z_f(s)$ independent of $a$. Under the local differentiability assumption that $\Qr^\pi(s, \cdot)$ is differentiable on $\mathrm{int}(\calA_f^\star(s))$, taking the gradient with respect to $a$ yields
\begin{equation*}
\nabla_a \log \pi^\star(a \mid s) \;=\; \alpha_r\, \nabla_a \Qr^\pi(s, a), \qquad s \in \calS_f^\star,\; a \in \mathrm{int}(\calA_f^\star(s)).
\end{equation*}
For $s \notin \calS_f^\star$ and $a \in \mathrm{int}(\calA)$, $\log \pi^\star(a \mid s) = -\beta\, \Qh^\star(s, a) - \log Z_{\mathrm{rec}}(s)$, so under differentiability of $\Qh^\star(s, \cdot)$ on $\mathrm{int}(\calA)$,
\begin{equation*}
\nabla_a \log \pi^\star(a \mid s) \;=\; -\beta\, \nabla_a \Qh^\star(s, a).
\end{equation*}
The condition $\Qh^\star(s, a) < 0$ used in the theorem statement is sufficient for $a \in \mathrm{int}(\calA_f^\star(s))$, since the strict sublevel set $\{\Qh^\star(s, \cdot) < 0\}$ is open whenever $\Qh^\star(s, \cdot)$ is continuous; we adopt this form in the main text because it is directly checkable from the routing gate $\Iind$.
\end{proof}

\subsubsection*{Zero-Temperature Concentration (sanity check)}

For completeness, we record the zero-temperature behaviour of the target policy $\pi^\star$ in~\eqref{eq:pi-gibbs}, summarised in Section~\ref{sec:gibbs} of the main text. This complements Theorem~\ref{thm:two-end-score} by showing that the entropy regularization is a continuous relaxation of the unregularized greedy subproblems on $(\Qr^\pi, \Qh^\star)$.

\begin{proposition}[Zero-temperature concentration]
\label{prop:zero-temp}
Suppose, in addition to the assumptions of Theorem~\ref{thm:gibbs-two-end}, that $\Qr^\pi(s, \cdot)$ and $\Qh^\star(s, \cdot)$ are continuous on $\calA$, and that every open neighbourhood of $\argmax_{a \in \calA_f^\star(s)} \Qr^\pi(s, a)$ (resp.\ $\argmin_{a \in \calA} \Qh^\star(s, a)$) has positive Lebesgue measure within the corresponding support. As $\alpha_r \to \infty$, every weak limit point of $\pi^\star(\cdot \mid s)$ for $s \in \calS_f^\star$ is supported on $\argmax_{a \in \calA_f^\star(s)} \Qr^\pi(s, a)$; symmetrically, as $\beta \to \infty$, $\pi^\star$ for $s \notin \calS_f^\star$ concentrates on $\argmin_{a \in \calA} \Qh^\star(s, a)$. These limits recover the unregularized maximizers of~\eqref{eq:pi-feasible}--\eqref{eq:pi-infeasible}.
\end{proposition}

\begin{proof}
Fix $s \in \calS_f^\star$ and let $Q^\star = \max_{a \in \calA_f^\star(s)} \Qr^\pi(s, a)$, which is attained by continuity of $\Qr^\pi(s, \cdot)$ on the compact set $\calA_f^\star(s)$. Fix $\eta > 0$, and define the two superlevel sets
\begin{equation*}
A_\eta \;\triangleq\; \{a \in \calA_f^\star(s) : \Qr^\pi(s, a) \geq Q^\star - \eta\},
\qquad
A_{\eta / 2} \;\triangleq\; \{a \in \calA_f^\star(s) : \Qr^\pi(s, a) \geq Q^\star - \eta / 2\}.
\end{equation*}
Both are closed by continuity of $\Qr^\pi(s, \cdot)$, and $A_{\eta / 2} \subseteq A_\eta \subseteq \calA_f^\star(s)$. Under the Boltzmann density $\pi^\star(\cdot \mid s)$ from~\eqref{eq:proof-feasible-candidate}, the probability of sampling an $\eta$-suboptimal action is
\begin{equation*}
\Pr_{a \sim \pi^\star(\cdot \mid s)}\!\big(\Qr^\pi(s, a) < Q^\star - \eta\big)
\;=\; \frac{\int_{\calA_f^\star(s) \setminus A_\eta} \exp(\alpha_r\, \Qr^\pi(s, a))\, da}{\int_{\calA_f^\star(s)} \exp(\alpha_r\, \Qr^\pi(s, a))\, da}.
\end{equation*}
To obtain a non-trivial rate as $\alpha_r \to \infty$, we upper-bound the numerator on $\calA_f^\star(s) \setminus A_\eta$ using the sublevel constraint $\Qr^\pi(s, a) < Q^\star - \eta$, and lower-bound the denominator on the smaller set $A_{\eta / 2} \subseteq \calA_f^\star(s)$ using the superlevel constraint $\Qr^\pi(s, a) \geq Q^\star - \eta / 2$. This creates a positive gap of $\eta / 2$ in the exponent:
\begin{align*}
\Pr_{a \sim \pi^\star(\cdot \mid s)}\!\big(\Qr^\pi(s, a) < Q^\star - \eta\big)
&\leq \frac{\lambda(\calA_f^\star(s)) \cdot \exp(\alpha_r (Q^\star - \eta))}{\lambda(A_{\eta / 2}) \cdot \exp(\alpha_r (Q^\star - \eta / 2))} \\
&= \frac{\lambda(\calA_f^\star(s))}{\lambda(A_{\eta / 2})} \cdot \exp(-\alpha_r\, \eta / 2).
\end{align*}
Under the regularity assumption that every open neighbourhood of $\argmax_{a \in \calA_f^\star(s)} \Qr^\pi(s, a)$ has positive Lebesgue measure within $\calA_f^\star(s)$, continuity of $\Qr^\pi(s, \cdot)$ implies that $A_{\eta / 2}$ contains an open neighbourhood of the argmax set, so $\lambda(A_{\eta / 2}) > 0$. The ratio on the right therefore tends to $0$ as $\alpha_r \to \infty$. Since this holds for every $\eta > 0$, $\pi^\star(U \mid s) \to 1$ for every open neighbourhood $U$ of $\argmax_{a \in \calA_f^\star(s)} \Qr^\pi(s, a)$; equivalently, every weak limit point of $\pi^\star(\cdot \mid s)$ is supported on this argmax set. Thus, in the zero-temperature limit $\alpha_r \to \infty$, $\pi^\star(\cdot \mid s)$ concentrates on the unregularized maximizers of~\eqref{eq:pi-feasible}. The infeasible-state branch for $s \notin \calS_f^\star$ is analogous, with $-\Qh^\star$ in place of $\Qr^\pi$, $\calA$ in place of $\calA_f^\star(s)$, and $\beta$ in place of $\alpha_r$; under continuity of $\Qh^\star(s, \cdot)$ on $\calA$ together with the analogous regularity at $\argmin_{a \in \calA} \Qh^\star(s, a)$, $\pi^\star(\cdot \mid s)$ concentrates on this argmin set as $\beta \to \infty$, recovering the greedy HJ-based safety optimizer of~\eqref{eq:pi-infeasible}.
\end{proof}

\subsection{Denoising Regression Identity and DDPM Guidance Coefficient}
\label{app:proof-regression}

This appendix establishes the denoising-regression identity referenced in Section~\ref{sec:score} and computes the DDPM guidance coefficient.

\paragraph{Population $L^2$ optimum.}
The loss~\eqref{eq:loss} is a square regression of $\epsilon_\theta(s, a_t, t)$ on $-M_q\, \bar\varphi(s, a_t)$ under the joint distribution of $(s, a_t, t)$ induced by $(s, a_0) \sim \calB$, $t \sim \mathrm{Unif}\{1, \ldots, T\}$, $\epsilon \sim \mathcal{N}(0, I)$, and $a_t = \sqrt{\bar\alpha_t}\, a_0 + \sigma(t)\, \epsilon$. The population $L^2$-optimum over measurable functions is the conditional expectation
\begin{equation*}
\epsilon_\theta^\star(s, a_t, t) = \E\!\big[\,{-}M_q\, \bar\varphi(s, a_t) \,\big|\, s, a_t, t\,\big].
\end{equation*}
Since $-M_q\, \bar\varphi(s, a_t)$ is a measurable function of $(s, a_t)$ alone (it does not depend on the latent $(a_0, \epsilon)$ that produced $a_t$, nor on $t$), the conditional expectation reduces to the quantity itself: $\epsilon_\theta^\star(s, a_t, t) = -M_q\, \bar\varphi(s, a_t)$, independent of $t$.

\paragraph{Guidance-component alignment.}
Substituting $\epsilon_\theta = \epsilon_\theta^\star = -M_q\, \bar\varphi$ into the DDPM reverse update~\eqref{eq:ddpm_reverse} gives
\begin{equation*}
a_{t - 1} = \Pi_\calA\!\!\left(\frac{1}{\sqrt{\alpha_t}}\!\left(a_t + \frac{\beta_t}{\sqrt{1 - \bar\alpha_t}}\, M_q\, \bar\varphi(s, a_t)\right) + \tilde\sigma_t\, z\right),
\qquad z \sim \mathcal{N}(0, I).
\end{equation*}
Collecting the $\bar\varphi$-dependent contribution, the reverse update adds to $a_t$ a guidance term
\begin{equation}
\label{eq:guidance-coef}
\Delta_{\mathrm{guide}}(s, a_t, t) = \frac{M_q\, \beta_t}{\sqrt{\alpha_t}\, \sqrt{1 - \bar\alpha_t}}\, \bar\varphi(s, a_t),
\end{equation}
with a positive, time-dependent coefficient $M_q\, \beta_t / (\sqrt{\alpha_t}\, \sqrt{1 - \bar\alpha_t}) > 0$. The remaining terms---the base rescaling $(1 / \sqrt{\alpha_t} - 1)\, a_t$, the Gaussian noise $\tilde\sigma_t\, z$, and the projection $\Pi_\calA$---do not depend on $\bar\varphi$. The decomposition~\eqref{eq:guidance-coef} is exact \emph{before projection}, and after projection only on the region where $\Pi_\calA$ is locally inactive; at action-boundary points where $\Pi_\calA$ is active the guidance contribution can be attenuated or eliminated. We therefore obtain guidance-component alignment with the clean-energy score, not an exact correspondence between the finite-step reverse chain and the score field of $\pi^\star$. On the interior of $\operatorname{supp}\pi^\star$ and away from action boundaries, Theorem~\ref{thm:two-end-score} gives $\bar\varphi = \nabla_a \log \pi^\star$, so $\Delta_{\mathrm{guide}}$ is proportional to the clean-energy score of the target policy $\pi^\star$ through the time-dependent coefficient in~\eqref{eq:guidance-coef}.
\hfill$\square$

\section{Implementation and Pseudo Code}
\label{app:impl-algo}
\label{app:algorithm} %

This appendix collects the algorithmic details deferred from
Section~\ref{sec:practical-algorithm}: the full SSM training loop, the
reward critic and denoiser updates, and an expectile-smoothed
alternative to the finite-candidate $\Vh$ estimator.

\paragraph{Pseudo code.}
Algorithm~\ref{alg:ssm-full} gives the full online SSM training loop
corresponding to Section~\ref{sec:practical-algorithm}. The algorithm
is written at the level of method-defining operations: replay
collection, finite-candidate HJ minimization, reward and safety
critic updates, and the HJ-gated denoising regression.

\begin{algorithm}[H]
\caption{Safe Score Matching (SSM)}
\label{alg:ssm-full}
\begin{algorithmic}[1]
\REQUIRE Replay buffer $\calB$; DDPM steps $T$ and noise schedule $\{\alpha_t,\bar\alpha_t,\beta_t\}_{t=1}^T$; safety discount $\gamma_h$; reward discount $\gamma_r$; score strength $M_q$; reward inverse temperature $\alpha_r$; recovery inverse temperature $\beta$; total candidates $K$; proposal scale $\sigma_\eta$.
\STATE Initialize denoiser $\epsilon_\theta$, safety critic $Q_{h,\psi}$, reward critics $Q_{r,1},Q_{r,2}$, and EMA target networks.
\STATE Collect random transitions and store $(s,a,r,h(s),s')$ in $\calB$.
\FOR{each environment step}
    \STATE Sample action $a_0 \sim \pi_\theta(\cdot \mid s)$ by the DDPM reverse update~\eqref{eq:ddpm_reverse}; execute $a_0$ and add the new transition to $\calB$.
    \FOR{each gradient update}
        \STATE Sample a minibatch $\{(s_i,a_i,r_i,h_i,s_i')\}_{i=1}^B$ from $\calB$.
        \STATE For each $s_i'$, draw a single policy sample $a_\theta^{(i)} \sim \pi_\theta(\cdot \mid s_i')$ and form the candidate set
        \[
        \calC_K(s_i') \;=\; \{a_\theta^{(i)}\} \;\cup\; \{\Pi_\calA(a_\theta^{(i)} + \sigma_\eta\, \eta^{(j)})\}_{j=1}^{K-1},
        \qquad \eta^{(j)} \sim \mathcal N(0, I).
        \]
        \STATE Estimate the target-state HJ value by finite-candidate minimization with the EMA target safety critic $\bar\Qh$:
        \[
        \widehat{\bar \Vh}(s_i')=\min_{a'\in\calC_K(s_i')}\bar\Qh(s_i',a').
        \]
        \STATE Update $\Qh$ toward
        \[
        y_{h,i}=(1-\gamma_h)h_i+\gamma_h\max\{h_i,\widehat{\bar \Vh}(s_i')\}.
        \]
        \STATE Sample $a_i' \sim \pi_\theta(\cdot \mid s_i')$ and update $Q_{r,1},Q_{r,2}$ toward
        \[
        y_{r,i}=r_i+\gamma_r\min_{j\in\{1,2\}}\bar Q_{r,j}(s_i',a_i').
        \]
        \STATE Sample diffusion index $t_i\sim\mathrm{Unif}\{1,\ldots,T\}$ and noise $\epsilon_i\sim\mathcal{N}(0,I)$; construct
        \[
        a_{t_{i}}=\sqrt{\bar\alpha_{t_{i}}}\,a_i+\sqrt{1-\bar\alpha_{t_{i}}}\,\epsilon_i.
        \]
        \STATE Compute the current-state gate with the online safety critic $\Qh$:
        \[
        \widehat{\Vh}(s_i)=\min_{a\in\calC_K(s_i)}\Qh(s_i,a),
        \]
        \[
        \Isafe^{(i)}=\indicator{\widehat{\Vh}(s_i)\leq 0},
        \qquad
        \Iind^{(i)}=\indicator{\Qh(s_i,a_{t_{i}})\leq0}.
        \]
        \STATE Compute the HJ-gated score target
        \[
        \bar\varphi_i
        =
        \Isafe^{(i)}\Iind^{(i)}\,\alpha_r\nabla_{a_{t_{i}}}\min_j Q_{r,j}(s_i,a_{t_{i}})
        +(1-\Isafe^{(i)})(-\beta\nabla_{a_{t_{i}}}\Qh(s_i,a_{t_{i}})).
        \]
        \STATE Update the denoiser by
        \[
        \nabla_\theta
        \frac{1}{B}
        \sum_{i=1}^B
        \left\|
        \epsilon_\theta(s_i,a_{t_{i}},t_i)
        -
        \operatorname{sg}[-M_q\bar\varphi_i]
        \right\|^2.
        \]
        \STATE Update all EMA target networks.
    \ENDFOR
\ENDFOR
\end{algorithmic}
\end{algorithm}

\paragraph{Reward critic and denoiser.}
$\Qr$ is trained by clipped double-$Q$
regression~\citep{fujimoto2018addressing} with bootstrap actions
sampled from the current diffusion policy. Given learned
$\Vh, \Qh, \Qr$, the score target $\bar\varphi$
in~\eqref{eq:score_target} is computed by a forward pass through the
critics and one autograd call for the action gradient, and the
denoiser is trained by minimizing~\eqref{eq:loss}. The reverse
chain~\eqref{eq:ddpm_reverse} produces both environment actions and
$\Qr$ bootstrap actions, so SSM operates in an off-policy
actor--critic regime.

\paragraph{Guidance normalization and the non-viable set.}
On the control tasks, the implementation normalizes the reward gradient separately for each sample, $\alpha_r \gqr / (\|\gqr\| + c_0)$ with $c_0 = 10^{-8}$, and uses the raw recovery gradient $-\beta \gqh$. When the state is classified as feasible but the queried action is not, where~\eqref{eq:score_target} sets $\bar\varphi = 0$, it uses $-\alpha_r \gqh / (\|\gqh\| + c_0)$, which guides the action toward lower predicted violation. The normalization rescales the reward branch without changing its direction, and the second term acts only where the score of $\pi^\star$ is undefined. The implemented target therefore agrees in direction with Theorem~\ref{thm:two-end-score} wherever that score is defined.

\paragraph{Expectile-smoothed $\Vh$ estimator.}
\label{app:expectile}
A complementary surrogate for $\Vh(s) = \min_{a \in \calA} \Qh(s, a)$,
used by FISOR~\citep{zheng2024safe} and motivated by expectile
regression~\citep{kostrikov2022offline, newey1987asymmetric}, fits a
separate state-value network $\Vh^\tau$ to a smoothed lower expectile
of $\Qh(s, \cdot)$ over the replay-buffer action distribution. For
$\tau \in (0, \tfrac{1}{2})$, the reversed expectile loss is
\begin{equation}
\label{eq:expectile-loss}
L_\tau^{\mathrm{rev}}(u) \;=\; \big|\tau - \indicator{u < 0}\big| \cdot u^2,
\end{equation}
which down-weights residuals above zero and up-weights those below.
The expectile state-value network is fit by
\begin{equation}
\label{eq:expectile-vh}
\calL_{\Vh^\tau}
\;=\;
\E_{(s, a) \sim \calB}\!\Big[
\, L_\tau^{\mathrm{rev}}\!\big(\Qh(s, a) - \Vh^\tau(s)\big) \,
\Big],
\end{equation}
and $\Vh^\tau$ is substituted for $\widehat \Vh$ in the feasibility
gate of~\eqref{eq:score_target}. We use $\tau = 0.1$ on the Quad2D
trajectory-tracking benchmark, whose reference repeatedly approaches
the safety boundary; all other tasks use the finite-candidate
estimator~\eqref{eq:vh-mc}. The score-matching framework is agnostic
to the choice of estimator: any consistent surrogate for
$\min_{a} \Qh(s, a)$ is admissible.

\section{More Experimental Results}
\label{app:experiments}

\subsection{Network Architecture}
\label{app:architecture}

The safety critic $\Qh$ and the two reward critics $Q_{r,1}, Q_{r,2}$ are
2-hidden-layer MLPs of width $(512, 512)$ on the control tasks; Table~\ref{tab:velocity-hparams} lists the velocity-task networks.
The denoiser $\epsilon_\theta$ uses Mish activations, a learned Fourier
time embedding (dim $= 64$), and a separate observation encoder MLP
$(128, 128, \text{Mish})$; its hidden widths are task-dependent, with
$(512, 512)$ for Quad2D and $(512, 512, 512)$ for the higher-dimensional
Quad3D and F16 tasks (Table~\ref{tab:hyperparams-task}). The state value $\Vh$ is obtained from $\Qh$ by the
finite-candidate estimator~\eqref{eq:vh-mc}, except on Quad2D, which
trains the expectile value network of Appendix~\ref{app:expectile}.

\subsection{Hyperparameters}
\label{app:hyperparams}

Table~\ref{tab:hyperparams-shared} lists the SSM hyperparameters that
are shared across all three benchmarks; Table~\ref{tab:hyperparams-task}
lists the task-specific training settings. The corresponding system
dynamics, reward functions, and safety functions are specified in
Appendix~\ref{app:dynamics}.

\begin{table}[h]
\centering
\small
\caption{Shared SSM hyperparameters used on all three benchmarks
(Quad2D, Quad3D, F16).}
\label{tab:hyperparams-shared}
\begin{tabular}{lcc}
\toprule
Hyperparameter & Symbol & Value \\
\midrule
Diffusion denoising steps      & $T$            & $5$ \\
Diffusion noise schedule       & ---            & VP \\
DDPM evaluation temperature    & ---            & $0.2$ \\
Score-field strength           & $M_q$          & $10$ \\
Reward-gradient coefficient    & $\alpha_r$     & $1$ \\
Recovery coefficient           & $\beta$        & $5$ \\
Safety margin                  & $\delta$       & $0$ \\
Safety discount                & $\gamma_h$     & $0.999$ \\
EMA target rate                & $\tau$         & $0.005$ \\
Learning rate                  & ---            & $3 \times 10^{-4}$ (cosine decay) \\
Batch size                     & $B$            & $512$ \\
Updates per environment step   & $G$            & $1$ \\
Local proposal count           & $K$            & $8$ \\
Local proposal scale           & $\sigma_\eta$  & $0.3$ \\
\bottomrule
\end{tabular}
\end{table}

\begin{table}[h]
\centering
\small
\caption{Task-specific settings. The denoiser actor uses three hidden
layers on the higher-dimensional Quad3D and F16 tasks; the critic
architecture is shared across tasks (Appendix~\ref{app:architecture}).}
\label{tab:hyperparams-task}
\begin{tabular}{lccccc}
\toprule
Task & Actor hidden & Training steps & Warmup & Episode horizon & $\gamma_r$ \\
\midrule
Quad2D tracking         & $(512, 512)$       & $2 \times 10^6$  & $5{,}000$   & $360$  & $0.99$  \\
Quad3D stabilization    & $(512, 512, 512)$  & $1 \times 10^6$  & $5{,}000$   & $500$  & $0.99$  \\
F16 stabilization       & $(512, 512, 512)$  & $2 \times 10^6$  & $20{,}000$  & $640$  & $0.995$ \\
\bottomrule
\end{tabular}
\end{table}

\subsection{System Dynamics and Safety Functions}
\label{app:dynamics}

This appendix specifies the dynamics, reward, and safety function
$h(s)$ for each of the three benchmarks. All three are deterministic
continuous-time systems that we discretize at a fixed step size; the
integrator and step size are listed per task.

\paragraph{Quad2D trajectory tracking.}
The state is
$s = (x, v_x, z, v_z, \theta, \omega) \in \R^6$ and the action is
$a \in [-1, 1]^2$, mapped to two thrust commands $u_i = (a_i + 1)/2$
with rotor forces $F_i = m g u_i$. The continuous-time dynamics are
\begin{align}
\label{eq:quad2d-ode}
\dot x      &= v_x,
&
\dot z      &= v_z,
&
\dot \theta &= \omega, \notag \\
\dot v_x    &= -\frac{(F_1 + F_2)\sin\theta}{m},
&
\dot v_z    &= \frac{(F_1 + F_2)\cos\theta}{m} - g,
&
\dot \omega &= \frac{0.1\,(u_2 - u_1)}{I},
\end{align}
with $m = 1$, $I = 0.02$, $g = 9.81$, and step size $\Delta t = 1/60$;
we integrate with semi-implicit Euler. The reward tracks a circular
reference and penalizes state error and control effort,
\begin{equation}
r(s, a) \;=\; -(s - s_{\text{ref}})^\top Q (s - s_{\text{ref}}) \;-\; 10^{-3} \|a\|_2^2,
\qquad
Q = \mathrm{diag}(10,\, 1,\, 10,\, 1,\, 0.2,\, 0.2),
\end{equation}
and the safety function combines a tight altitude band with loose
position bounds,
\begin{equation}
h(s) \;=\; \max\!\big(0.5 - z,\;\; z - 1.5,\;\; |x| - 2,\;\; |z| - 3\big).
\end{equation}

\paragraph{Quad3D regulation.}
We use the standard nine-state quadrotor model. The state is
$s = (p_x, p_y, p_z, v_x, v_y, v_z, \phi, \theta, \psi) \in \R^9$ with
$p_z$ measured positive downward, and the action $a \in [-1, 1]^4$ is
mapped to a thrust command and three angular-rate commands,
\begin{equation}
\text{thrust} \;=\; m g\, (1 + a_0),
\qquad
(\dot\phi_{\text{cmd}},\; \dot\theta_{\text{cmd}},\; \dot\psi_{\text{cmd}}) \;=\; 5\,(a_1, a_2, a_3).
\end{equation}
The dynamics are
\begin{align}
\label{eq:quad3d-ode}
\dot p_x &= v_x,
&
\dot p_y &= v_y,
&
\dot p_z &= v_z, \notag \\
\dot v_x &= -\frac{\text{thrust}\,\sin\theta}{m},
&
\dot v_y &= \frac{\text{thrust}\,\cos\theta\,\sin\phi}{m},
&
\dot v_z &= g - \frac{\text{thrust}\,\cos\theta\,\cos\phi}{m}, \notag \\
\dot \phi &= \dot\phi_{\text{cmd}},
&
\dot \theta &= \dot\theta_{\text{cmd}},
&
\dot \psi &= \dot\psi_{\text{cmd}},
\end{align}
with $m = 1$, $g = 9.80665$, and step size $\Delta t = 0.01$; we
integrate with forward Euler. The reward is a weighted $\ell_1$
regulation toward the origin (target $p_z = 0$) with weights
$q_{\text{pos}} = 15$, $q_{\text{vel}} = 6$, $q_{\text{ang}} = 0.5$,
$q_\psi = 2$, and no action penalty. The safety function imposes a
ground constraint and a bounded workspace,
\begin{equation}
h(s) \;=\; \max\!\big(p_z,\;\; \|p\|_2 - 3\big).
\end{equation}

\paragraph{F16 stabilize-and-avoid.}
We adopt the F16 stabilize-and-avoid setup of~\citet{so2023solving},
based on the standard 16-state F16 aircraft model. The state is
\begin{equation*}
x \;=\; (V_T,\, \alpha,\, \beta,\, \phi,\, \theta,\, \psi,\, p,\, q,\, r,\, p_N,\, p_E,\, H,\, \mathrm{pow},\, n_{z\mathrm{int}},\, p_{s\mathrm{int}},\, n_{y_r\mathrm{int}}),
\end{equation*}
and the four-dimensional action $a \in [-1, 1]^4$ controls three control
surfaces and a throttle; physical commands are obtained by an affine
remapping $u = \tfrac{1}{2}(u_{\max} - u_{\min})\,a + \tfrac{1}{2}(u_{\max} + u_{\min})$
with $u_{\min} = (-10, -10, -10, 0)$ and $u_{\max} = (15, 10, 10, 1)$.
We integrate the standard F16 nonlinear dynamics
$\dot x = F_{\text{F16}}(x, u)$ with RK4 at step size
$\Delta t = 0.05$. The task is to stabilize the aircraft into a
low-altitude target band $H \in [50, 150]$ for at least $50$
consecutive steps; the reward penalizes deviation from this band with
a piecewise-linear loss and adds a small bonus once the trajectory
enters it. The safety function aggregates six per-axis violations,
\begin{equation}
h(x) \;=\; \max\!\big(h_{\text{alt}},\;\; h_\alpha,\;\; h_\beta,\;\; h_\theta,\;\; h_{p_E},\;\; h_p\big),
\end{equation}
where each component normalizes the gap between the current state and
the corresponding bound. The active bounds are altitude
$H \in [0,\, 1000]$, angle of attack
$\alpha \in [-0.1745,\, 0.7854]$, sideslip $|\beta| \le 0.5236$,
pitch $|\theta| \le 0.95\,\pi/2$, lateral position $|p_E| \le 200$,
and roll rate $|p| \le 8$ (radians or meters as appropriate).

\subsection{Evaluation Protocol and Metrics}
\label{app:eval-protocol}

On the control tasks, every baseline is rerun in our setup and trained for its recommended number of iterations or for at least three hours, and every method is evaluated at its best checkpoint. SSM runs on one NVIDIA RTX 5090 GPU, and the baselines run on RTX 4090 and RTX 5090 GPUs.

The Quad2D tracking error sums the squared errors of the six state coordinates at each step, averages them over the steps of an episode, and then averages over episodes, including those that end in a crash. The Quad3D terminal tracking $\ell_1$ is
\begin{equation}
\label{eq:term-l1}
E_{\mathrm{term}} = \frac{1}{N} \sum_{n = 1}^{N} \big\| s_T^{(n)} - s_{\mathrm{ref}} \big\|_1,
\end{equation}
where $N$ is the number of evaluation episodes, including those that end in a crash. The final-50-step tracking $\ell_1$ averages the same distance over the last 50 steps of each episode. The Quad3D reward is a weighted $\ell_1$ penalty (Appendix~\ref{app:dynamics}), which keeps penalizing small residual errors near the target; the quadratic tracking reward applies only to Quad2D. The F16 stabilization rate is the fraction of episodes that keep the altitude in the band $H \in [50, 150]$ for at least 50 consecutive steps. Section~\ref{sec:experiments} defines the safety metrics~\eqref{eq:safety-metrics}.

\subsection{Quad3D Ablation Study}
\label{app:ablation-quad3d}

This appendix expands the Quad3D ablation summarized in
Section~\ref{sec:practical-algorithm}. We add (i) two structural
baselines beyond the candidate-set sweep --- an expectile-regressed
$\Vh^\tau$ baseline that replaces the sampled gate~\eqref{eq:vh-mc}
with a separate value network (Appendix~\ref{app:expectile}), and a
single-action gate ($K = 1$) that drops local proposals altogether ---
and (ii) the full $K \times \sigma_\eta$ grid. All variants share the
anchor configuration of the main-text Quad3D entry
(Section~\ref{sec:experiments}); each variant modifies only the
indicated knob.

\paragraph{Evaluation.}
We evaluate each variant on $500$ initial states sampled from the
Quad3D reset box with rejection on $h(s) < 0$ (non-goal and
physically safe; mean initial $h$: $-0.530$, $95$th percentile:
$-0.099$). We report the terminal tracking $\ell_1$~\eqref{eq:term-l1}, which
includes crashed episodes, as the task metric, since it is more
aligned with smooth stabilization than thresholded success. We
report \emph{episode violation fraction} (the fraction of episodes
with at least one valid step satisfying $h(s_t) > 0$) and
\emph{step violation rate} (the fraction of valid steps with
$h(s_t) > 0$) as safety metrics, alongside \emph{crash fraction}
(hard-termination rate).

\paragraph{Master ablation table.}
Table~\ref{tab:ablation-full} consolidates all 12 variants on the
500-state evaluation set. The shaded row marks the canonical SSM
configuration; boldface marks the column-wise best within the
sampled-$\Qh$ block.

\begin{table}[h]
\centering
\small
\caption{Full Quad3D ablation results ($n = 500$ evaluation episodes).
The top block reports estimator-level and candidate-source baselines;
the bottom block sweeps the local Gaussian-proposal grid
$K \in \{8, 16, 32\} \times \sigma_\eta \in \{0.1, 0.3, 0.8\}$.
Shaded row: the canonical SSM configuration used in the main results.
Boldface within the sampled-$\Qh$ block: the lowest terminal tracking $\ell_1$
(default $K = 8, \sigma_\eta = 0.3$) and the safety-clean cell
$(K, \sigma_\eta) = (32, 0.3)$, which is the only sampled-gate cell
that achieves zero observed violation across all three safety metrics
at the lowest accompanying $\ell_1$.}
\label{tab:ablation-full}
\begin{tabular}{lccrrrr}
\toprule
Variant & $K$ & $\sigma_\eta$
& Term.\ $\ell_1\!\downarrow$ & Crash\,$\downarrow$
& Ep.\ vio.\,$\downarrow$
& Step vio.\,$\downarrow$ \\
&&&&&& {\scriptsize ($\times 10^{-3}$)} \\
\midrule
\multicolumn{7}{l}{\textit{Structural ablations}} \\
Expectile-$\Vh^\tau$ baseline    & ---  & ---   & $3.596$          & $0.002$          & $0.004$          & $0.09$          \\
Policy-only candidates           & $8$  & ---   & $1.793$          & $0.006$          & $0.006$          & $0.65$          \\
Single-action gate ($K = 1$)     & $1$  & ---   & $0.922$          & $0.000$          & $0.000$          & $0.00$          \\
\midrule
\multicolumn{7}{l}{\textit{Sampled-$\Qh$ gate: candidate-count $\times$ proposal-scale grid}} \\
Sampled-$\Qh$ gate               & $8$  & $0.1$ & $0.650$          & $0.016$          & $0.016$          & $1.60$          \\
\rowcolor{anchorrow}
Sampled-$\Qh$ gate (default)     & $8$  & $0.3$ & $\mathbf{0.321}$ & $0.002$          & $0.022$          & $2.10$          \\
Sampled-$\Qh$ gate               & $8$  & $0.8$ & $1.233$          & $0.008$          & $0.008$          & $0.80$          \\
Sampled-$\Qh$ gate               & $16$ & $0.1$ & $1.131$          & $0.016$          & $0.028$          & $6.10$          \\
Sampled-$\Qh$ gate               & $16$ & $0.3$ & $0.883$          & $0.002$          & $0.006$          & $0.40$          \\
Sampled-$\Qh$ gate               & $16$ & $0.8$ & $1.184$          & $0.000$          & $0.000$          & $0.00$          \\
Sampled-$\Qh$ gate               & $32$ & $0.1$ & $1.343$          & $0.000$          & $0.008$          & $0.60$          \\
Sampled-$\Qh$ gate (safety-clean) & $32$ & $0.3$ & $0.454$          & $\mathbf{0.000}$ & $\mathbf{0.000}$ & $\mathbf{0.00}$ \\
Sampled-$\Qh$ gate               & $32$ & $0.8$ & $2.353$          & $0.002$          & $0.008$          & $0.30$          \\
\bottomrule
\end{tabular}
\end{table}

\paragraph{Mechanism notes beyond the main text.}

\emph{Why the expectile baseline degrades on Quad3D.}
The expectile-regressed $\Vh^\tau$ inflates terminal tracking $\ell_1$ by an
order of magnitude over the default sampled gate
(Table~\ref{tab:ablation-full}). The two estimators aggregate over
actions in different ways: the sampled gate resolves
$\min_a \Qh(s, \cdot)$ locally on each state via the candidate set,
whereas the expectile network learns a smoothed envelope over the
policy's action distribution. Quad3D's wide reset distribution and
disconnected feasible region reward the local resolution.

\emph{Single-action vs.\ policy-only at the same no-Gaussian setting.}
Among the no-Gaussian variants, the single-action gate ($K = 1$)
reaches the lowest terminal tracking $\ell_1$ with zero observed violation,
while the policy-only setting ($K = 8$, no Gaussian) is the worst
configuration on terminal tracking $\ell_1$ (Table~\ref{tab:ablation-full}).
Both drop local proposals; only $K$ differs. A plausible explanation
is the bias of $\min$ over multiple samples drawn from the same
policy distribution: as $\pi_\theta$ concentrates around its modes,
$\widehat\Vh = \min_k \Qh(s, a_\theta^{(k)})$ becomes an optimistic
lower estimate of the policy's worst-case violation and routes more
states to the reward branch, where $\Qr$ is poorly calibrated.
Single-action avoids this by reading $\Qh$ at the anchor directly.
The $(K, \sigma_\eta) = (32, 0.3)$ configuration attains a lower terminal
tracking $\ell_1$ while retaining zero observed violations and crashes.

\emph{The $K$ trade-off does not extend off the $\sigma_\eta = 0.3$
ridge.} Along $\sigma_\eta = 0.3$, raising $K$ from 8 to 32 removes the
observed episode violations at a higher terminal tracking $\ell_1$; off
this ridge, the relationship between $K$ and episode violation is
non-monotonic (Table~\ref{tab:ablation-full}). A plausible explanation
is that proposals at $\sigma_\eta = 0.1$ stay too close to the policy
mode, while proposals at $\sigma_\eta = 0.8$ reach actions where $\Qh$
is poorly calibrated.

\subsection{Safety-Gymnasium Experiments and Additional Studies}
\label{app:safety-gym}

\paragraph{Tasks and evaluation protocol.}
\texttt{SafetyHalfCheetahVelocity-v1} and \texttt{SafetySwimmerVelocity-v1}~\citep{ji2023safety} emit a unit cost at every step whose forward velocity $v_x$ exceeds a task threshold $v_{\mathrm{thr}}$, $3.2096$ for HalfCheetah and $0.2282$ for Swimmer. The episodic cost counts these steps in an episode of 1000 steps. CAL and ALGD train against the budget $\E_\pi[C(\tau)] \le d$ with $d = 25$, which also serves as the feasibility criterion for every method; SSM, RAC, and RESPO optimize reachability objectives without a budget. Each method is trained with five seeds and evaluated without checkpoint selection at a fixed step: 1M environment steps, or 10M for RESPO, its native on-policy horizon (333 epochs of 30k steps). Each evaluation runs 50 episodes on CPU with shared reset seeds and one stochastic action sample per decision. RESPO keeps its own simulator versions and policy random-number stream, so its initial states need not match those of the other methods. Tables report the mean $\pm$ SD over seeds of per-seed means. These five-seed results, with the configurations specified here, supersede the preliminary velocity results reported during review.

\paragraph{SSM configuration.}
On the velocity tasks, SSM keeps the HJ critic, the finite-candidate gate of Section~\ref{sec:practical-algorithm}, and the reward/recovery structure of the target~\eqref{eq:pi-gibbs}. It trains the denoiser with a posterior-noise regression target instead of the gradient target~\eqref{eq:score_target}. We noise fresh samples from a lagged copy of the denoiser in pre-tanh coordinates ($a = \tanh u$) to obtain $u_t$, then estimate the posterior-mean noise $\E[\epsilon \mid s, u_t]$ under~\eqref{eq:pi-gibbs}, the noise-expectation form of the posterior-mean target in RFM~\citep{li2026reverse}, by self-normalized importance sampling over samples from a mixture proposal. The training label averages this estimate 1:1 with the lagged copy's prediction. In place of a fixed $\alpha_r$, we choose the reward coefficient so that the KL divergence between the reward-tilted and untilted sample weights matches a target value on average over the gated-feasible states. We reduce it separately on states whose KL would exceed a per-state cap. On HalfCheetah, the critic target~\eqref{eq:hj_bellman} replaces every positive $h$ by the constant $h_{\mathrm{gap}} = 0.5$, which keeps the sign of $h$ and the boundary $h = 0$; Swimmer uses the raw $h$. At test time, SSM acts with one diffusion sample without filtering by $\Qh$. Table~\ref{tab:velocity-hparams} lists the settings.

\begin{table}[h]
\centering
\small
\caption{SSM settings on the velocity tasks. Settings not listed follow Algorithm~\ref{alg:ssm-full}.}
\label{tab:velocity-hparams}
\begin{tabular}{@{}p{0.43\linewidth}p{0.53\linewidth}@{}}
\toprule
Setting & Value \\
\midrule
Actor; reward critics (twin); safety critic (single) & $3 \times 256$; $2 \times 256$; $2 \times 256$ \\
Diffusion steps $T$; schedule; DDPM temperature & $5$; VP; $0.2$, when acting and inside the update \\
Optimizer & Adam; actor $10^{-4}$ with global-norm clipping at $1.0$; critics $3 \times 10^{-4}$; constant rates \\
Batch size & critics $256$; actor $64$ \\
Discounts $\gamma_r$, $\gamma_h$; target EMA rate & $0.99$, $0.99$; $0.005$ \\
Gate candidates $K$; proposal scale $\sigma_\eta$ & $8$; $0.3$ \\
Recovery coefficient $\beta$ & $3$ \\
Importance samples & $64$: $32$ from the forward-noise likelihood and $32$ from a bridge around $8$ samples of the lagged denoiser \\
Reward-tilt KL: batch mean; per-state cap & $0.6$; $1.2$ \\
Lagged copy: refresh interval; averaging weight & $500$ actor updates; $0.5$ \\
$h_{\mathrm{gap}}$ & $0.5$ (HalfCheetah); $0$ (Swimmer) \\
Reward-critic loss & Huber ($\delta = 10$), targets clipped to $[-200, 1000]$ \\
Warm-up & random actions for $3 \times 10^4$ steps; critic updates from $2 \times 10^4$ steps, actor updates from $3 \times 10^4$ \\
Observation normalization & statistics of the first $10^4$ steps, then frozen \\
Updates per step; training steps & $1$; $10^6$ \\
\bottomrule
\end{tabular}
\end{table}

\paragraph{Baseline configurations.}
CAL uses an update-to-data ratio of 10, an ensemble of eight cost critics, and layer-normalized critics. ALGD uses an update-to-data ratio of 2 with the score-matching term of its release (weight 0.1). Both train with $d = 25$. Our HalfCheetah task differs from the HalfCheetah task in the ALGD paper~\citep{algd2026}, so the results are not directly comparable with those reported there. RESPO runs the official implementation and objective, with constraint threshold 0. RAC runs a JAX port of the official learner whose reference hyperparameters target 3M-step training, so we tuned it at the 1M budget and trained the selected configuration on the reported seeds. The selection followed a pre-specified rule and used only tuning seeds and a separate evaluation panel, although the candidate configurations were partly informed by earlier results on the reporting panel. On two of the five reported HalfCheetah seeds, the learned multiplier of RAC falls to about zero, and the policy then violates the constraint at almost every step. These two seeds account for the large mean and spread of RAC in Table~\ref{tab:safety-gym}.

\paragraph{Results throughout training.}
Table~\ref{tab:safety-gym-full} reports each method at 25, 50, 75, and 100\% of its horizon, and Figure~\ref{fig:velocity-curves} shows the learning curves. The cost of SSM peaks within the first 100k steps and remains low afterwards on both tasks. The reward of RESPO rises over its 10M-step horizon; read at the step counts of the other methods, it is far lower than theirs on both tasks.

\begin{table}[h]
\centering
\caption{Safety-Gymnasium results at 25, 50, 75, and 100\% of each method's horizon (1M steps; 10M for RESPO): episodic reward and cost, mean $\pm$ SD over five seeds. Asterisks mark values interpolated per seed between the two neighboring evaluated checkpoints. The row RESPO (0.25--1M) reads the RESPO runs at the step counts of the other methods.}
\label{tab:safety-gym-full}
\footnotesize
\renewcommand{\arraystretch}{0.92}
\setlength{\tabcolsep}{5pt}
\begin{tabular}{@{}llcccc@{}}
\toprule
Method & Metric & 25\% & 50\% & 75\% & 100\% \\
\midrule
\multicolumn{6}{@{}l}{\textit{HalfCheetah velocity}} \\
\multirow{2}{*}{SSM} & Reward & $2675\pm25$ & $2724\pm27$ & $2744\pm25$ & $2754\pm13$ \\
 & Cost & $0.0\pm0.0$ & $0.1\pm0.1$ & $0.0\pm0.0$ & $0.0\pm0.1$ \\
\multirow{2}{*}{RAC} & Reward & $1949\pm2579$ & $3356\pm2977$ & $3720\pm3456$ & $5014\pm3781$ \\
 & Cost & $195.6\pm437.3$ & $196.0\pm438.3$ & $196.5\pm438.5$ & $391.7\pm536.3$ \\
\multirow{2}{*}{CAL} & Reward & $2451\pm175$ & $2530\pm136$ & $2534\pm193$ & $2631\pm107$ \\
 & Cost & $15.8\pm4.2$ & $23.1\pm6.5$ & $22.2\pm11.7$ & $28.3\pm12.5$ \\
\multirow{2}{*}{ALGD} & Reward & $2533\pm134$ & $2629\pm113$ & $2677\pm103$ & $2695\pm99$ \\
 & Cost & $23.9\pm15.9$ & $16.1\pm8.3$ & $19.3\pm11.9$ & $42.0\pm39.7$ \\
\multirow{2}{*}{RESPO} & Reward & $1411\pm311^{*}$ & $1992\pm360^{*}$ & $2201\pm324^{*}$ & $2323\pm323$ \\
 & Cost & $6.9\pm3.7^{*}$ & $9.7\pm5.1^{*}$ & $11.4\pm5.4^{*}$ & $8.4\pm3.9$ \\
\multirow{2}{*}{RESPO (0.25--1M)} & Reward & $-178\pm26^{*}$ & $-2\pm57^{*}$ & $228\pm111^{*}$ & $447\pm142^{*}$ \\
 & Cost & $0.0\pm0.0^{*}$ & $0.1\pm0.1^{*}$ & $0.6\pm0.7^{*}$ & $1.6\pm1.3^{*}$ \\
\midrule
\multicolumn{6}{@{}l}{\textit{Swimmer velocity}} \\
\multirow{2}{*}{SSM} & Reward & $35\pm10$ & $42\pm8$ & $45\pm5$ & $45\pm6$ \\
 & Cost & $4.6\pm5.8$ & $1.2\pm1.0$ & $1.4\pm1.9$ & $0.5\pm0.4$ \\
\multirow{2}{*}{RAC} & Reward & $48\pm45$ & $68\pm57$ & $41\pm57$ & $75\pm57$ \\
 & Cost & $1.2\pm2.0$ & $6.2\pm6.4$ & $13.7\pm18.6$ & $14.5\pm28.4$ \\
\multirow{2}{*}{CAL} & Reward & $30\pm7$ & $32\pm6$ & $27\pm10$ & $31\pm7$ \\
 & Cost & $47.3\pm51.9$ & $14.8\pm13.1$ & $17.8\pm5.0$ & $27.9\pm21.3$ \\
\multirow{2}{*}{ALGD} & Reward & $42\pm2$ & $47\pm3$ & $48\pm5$ & $49\pm4$ \\
 & Cost & $3.6\pm1.5$ & $5.9\pm2.7$ & $5.9\pm3.5$ & $6.5\pm2.4$ \\
\multirow{2}{*}{RESPO} & Reward & $31\pm1^{*}$ & $34\pm1^{*}$ & $36\pm1^{*}$ & $37\pm1$ \\
 & Cost & $7.8\pm2.1^{*}$ & $7.3\pm1.4^{*}$ & $9.2\pm2.0^{*}$ & $8.4\pm1.3$ \\
\multirow{2}{*}{RESPO (0.25--1M)} & Reward & $-17\pm1^{*}$ & $-9\pm4^{*}$ & $2\pm7^{*}$ & $13\pm12^{*}$ \\
 & Cost & $4.2\pm2.6^{*}$ & $9.3\pm9.7^{*}$ & $8.1\pm5.2^{*}$ & $12.2\pm4.6^{*}$ \\
\bottomrule
\end{tabular}
\end{table}

\begin{figure}[h]
\centering
\includegraphics[width=\linewidth]{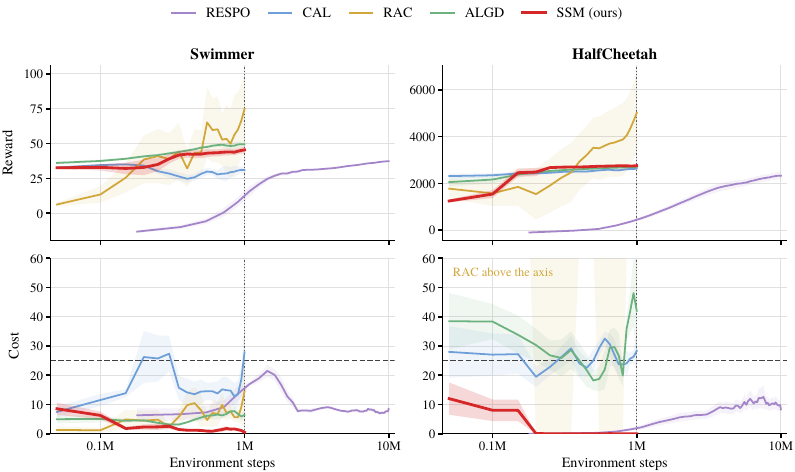}
\caption{Learning curves on the velocity tasks (mean $\pm 1$ standard error over five seeds; each seed smoothed with a centered moving average over 10\% of its horizon): test reward (top) and cost (bottom) against environment steps on a log scale. The dashed line marks the budget $d = 25$ and the dotted line the 1M endpoint. Cost axes are truncated; the mean cost of RAC on HalfCheetah stays between about 195 and 392 and lies above the axis. Tables use the unsmoothed evaluations.}
\label{fig:velocity-curves}
\end{figure}

\paragraph{Sensitivity to $\alpha_r$, $\beta$, and $M_q$.}
On the fixed-layout \texttt{SafetyCarGoal1-v0} task, we vary one coefficient at a time from $0.25\times$ to $4\times$ its default ($\alpha_r = 1$, $\beta = 5$, $M_q = 20$). Table~\ref{tab:sensitivity} reports the time average of each evaluation curve over the final 20\% of environment steps; $P(C_{\mathrm{ep}} = 0)$ is the fraction of evaluation episodes without violation in the same window. Across the full $16\times$ range of each coefficient, the default attains the highest reward and goal count, and every setting keeps a low cost and a high fraction of violation-free episodes. For given critics, the gates depend on $\widehat\Vh$ and $\Qh$ rather than on these coefficients: $\alpha_r$ scales the reward guidance in predicted-feasible states, $\beta$ the recovery guidance in predicted-infeasible states, and $M_q$ the whole denoising target. Changing them therefore rescales the guidance without moving the gates, consistent with the small variation in Table~\ref{tab:sensitivity}.

\begin{table}[h]
\centering
\small
\caption{Sensitivity of SSM to $\alpha_r$, $\beta$, and $M_q$ on \texttt{SafetyCarGoal1-v0}. The first row is the default; each other row changes one coefficient.}
\label{tab:sensitivity}
\begin{tabular}{@{}rrrrrrr@{}}
\toprule
$\alpha_r$ & $\beta$ & $M_q$ & Reward $\uparrow$ & Cost $\downarrow$ & Goals $\uparrow$ & $P(C_{\mathrm{ep}} = 0)$ $\uparrow$ \\
\midrule
1    & 5    & 20 & 36.10 & 1.00 & 18.65 & 93.5\% \\
0.25 & 5    & 20 & 35.39 & 0.52 & 18.19 & 96.7\% \\
4    & 5    & 20 & 36.03 & 1.05 & 18.57 & 93.3\% \\
1    & 1.25 & 20 & 35.99 & 0.83 & 18.59 & 95.2\% \\
1    & 20   & 20 & 35.99 & 0.97 & 18.39 & 92.7\% \\
1    & 5    & 5  & 35.55 & 0.61 & 18.33 & 94.3\% \\
1    & 5    & 80 & 35.36 & 0.81 & 18.30 & 93.2\% \\
\bottomrule
\end{tabular}
\end{table}

\paragraph{Policy class under an expected-cost formulation.}
To separate the effect of the policy class from that of HJ routing, we compare QSM-Lag and SAC-Lag on the same task (Table~\ref{tab:policy-class}). Both use a Lagrangian cost term; QSM-Lag trains a diffusion actor by Q-score matching~\citep{psenka2023learning}, and SAC-Lag trains a Gaussian actor~\citep{haarnoja2018sac, ray2019benchmarking}. At both budgets, QSM-Lag attains higher reward and goal count with substantially lower cost than SAC-Lag, and only QSM-Lag stays within the budget $d = 5$. Constrained navigation can produce several separated high-value action regions, which a Gaussian actor must cover with a single mode, whereas a diffusion actor can place mass on each~\citep{psenka2023learning, ding2024diffusion}.

\begin{table}[h]
\centering
\small
\caption{Diffusion versus Gaussian actors under the same expected-cost formulation on \texttt{SafetyCarGoal1-v0}, with the summary of Table~\ref{tab:sensitivity}.}
\label{tab:policy-class}
\begin{tabular}{@{}lrrrr@{}}
\toprule
Method & Budget $d$ & Reward $\uparrow$ & Cost $\downarrow$ & Goals $\uparrow$ \\
\midrule
QSM-Lag    & 5   & 36.25 & 1.58 & 18.73 \\
SAC-Lag    & 5   & 30.89 & 8.25 & 15.83 \\
QSM-Lag    & 0   & 36.34 & 1.27 & 18.90 \\
SAC-Lag    & 0   & 29.01 & 5.67 & 14.67 \\
SSM (ours) & --- & 36.10 & 1.00 & 18.65 \\
\bottomrule
\end{tabular}
\end{table}

\paragraph{Baseline properties.}
Table~\ref{tab:taxonomy} classifies the compared methods by whether they use a state-wise safety formulation, their policy class, and whether they update a dual variable. Among them, only SSM combines state-wise safety with a diffusion policy and no dual update.

\begin{table}[h]
\centering
\small
\caption{Properties of the compared methods.}
\label{tab:taxonomy}
\begin{tabular}{@{}lccc@{}}
\toprule
Method & State-wise safety & Policy class & Dual update \\
\midrule
CAL~\citep{wu2024offpolicy}           & No  & Gaussian  & Yes \\
RAC (RCRL)~\citep{yu2022reachability} & Yes & Gaussian  & Yes \\
RESPO~\citep{ganai2023iterative}      & Yes & Gaussian  & Yes \\
EFPPO~\citep{so2023solving}           & Yes & Gaussian  & No  \\
ALGD~\citep{algd2026}                 & No  & Diffusion & Yes \\
SSM (ours)                            & Yes & Diffusion & No  \\
\bottomrule
\end{tabular}
\end{table}

\paragraph{Computation and sample efficiency.}
On Quad3D, SSM reaches its reported performance within its training budget of $10^6$ environment steps, whereas RAC and RESPO require $4 \times 10^6$ and $6 \times 10^6$ steps to converge. CAL is the most sample-efficient method and converges within $5 \times 10^5$ steps. SSM trains in about one hour on Quad3D, and RESPO, the fastest baseline in wall-clock time, takes about 2.5 hours; the relative order is the same on the other tasks. The actor update differentiates only through the denoiser at one noise level rather than through the reverse chain. At deployment, SSM runs one reverse chain of $T = 5$ steps per action. The $K = 8$ gate candidates are used only during training, where they perturb a single policy sample and are evaluated as one critic batch.

\end{document}